\newif\iflong
\newif\ifshort

\longtrue 

\iflong
\else
\shorttrue
\fi

\documentclass[nohyperref]{article}

\usepackage{microtype}
\usepackage{graphicx}
\usepackage{subfigure,lipsum}
\usepackage{tikz}
\usetikzlibrary{shapes, positioning, patterns,decorations.pathreplacing}\tikzstyle{vertex}=[circle, draw, inner sep=0pt, minimum size=4pt,outer sep = 1pt]

\usepackage{booktabs} 

\usepackage{hyperref}

\usepackage[accepted]{icml2023}

\usepackage{amsmath}
\usepackage{amssymb}
\usepackage{mathtools}
\usepackage{amsthm}

\usepackage[capitalize,noabbrev]{cleveref}

\theoremstyle{plain}

\usepackage[textsize=tiny]{todonotes}

\usepackage{complexity}
\usepackage{enumitem}

\usepackage[utf8]{inputenc} 
\usepackage{hyperref}       
\usepackage{cleveref}

\usepackage{url}            
\usepackage{booktabs}       
\usepackage{amsfonts}       
\usepackage{nicefrac}       
\usepackage{microtype}      
\usepackage{xcolor}         
\usepackage{boxedminipage}
\usepackage{slashbox,diagbox}
\usepackage{xspace}
\usepackage[scaled=0.86]{helvet}

\renewcommand{\cc}[1]{{\mbox{\textnormal{\textsf{#1}}}}\xspace}  

\newcommand{\bigoh}{\mathcal{O}}    
\newcommand{\III}{\mathcal{I}}    
\newcommand{\Nat}{\mathbb{N}}

\newcommand{\SB}{\{\,}%
\newcommand{\SM}{\;{|}\;}%
\renewcommand{\SE}{\,\}}%
\newcommand{\VR}{V_\textnormal{R}}%
\newcommand{\VB}{V_\textnormal{B}}%

\newcommand{\RER}{R_\textnormal{R}}%
\newcommand{\REB}{R_\textnormal{B}}%

\newcommand{\Weft}{{\cc{W}}}
\renewcommand{\W}[1]{{\Weft}{\normalfont[#1]}}
\newcommand{\paraNP}{\cc{paraNP}}

\newcommand{\YES}{\cc{Yes}}
\newcommand{\NO}{\cc{No}}

\newcommand{\red}{red}
\newcommand{\blue}{blue}

\newcommand{\pastCenter}{c_{\textnormal{past}}}
\newcommand{\futureCenter}{c_{\textnormal{future}}}
\newcommand{\bagCenter}{C_{\textnormal{bag}}}
\newcommand{\bagVectors}{V_\textnormal{bag}}
\newcommand{\noInst}{\bot}

\newcommand{\tw}{\textsf{tw}}

\newcommand{\incgraph}{G_I}

\newtheorem{theorem}{Theorem}

\newtheorem{observation}[theorem]{Observation}
\newtheorem{proposition}[theorem]{Proposition}

\newtheorem{corollary}[theorem]{Corollary}
\newtheorem{lemma}[theorem]{Lemma}

\newcommand{\pbDef}[3]{%
\noindent
\begin{center}
\begin{boxedminipage}{1\columnwidth}
#1\\[5pt]
\begin{tabular}{l p{0.75\columnwidth}}
Input: & #2\\
Question: & #3
\end{tabular}
\end{boxedminipage}
\end{center}
}

\newcommand{\CCC}{\mathcal{C}}
\newcommand{\tuple}[1]{\langle{#1}\rangle}  

\newcommand{\CSP}{\mathrm{CSP}}

\newcommand{\vO}{\textbf{1}}

\newcommand{\CRB}{\textsc{EC-BHC}}
\newcommand{\scp}{\textsf{\textup{econ}}}
\newcommand{\scpm}{\textsf{econ}_{min}}
\newcommand{\scptw}{\lambda}
\newcommand{\icon}{\textsf{\textup{dcon}}}
\newcommand{\con}{\textsf{\textup{con}}}
\newcommand{\UHS}{\textsc{Uniform Hitting Set}}
\newcommand{\FFF}{\mathcal{F}}
\newcommand{\yes}{\YES}
\newcommand{\no}{\NO}

\begin{document}
\icmltitlerunning{The Computational Complexity of Concise Hypersphere Classification}
\twocolumn[
\icmltitle{The Computational Complexity of Concise Hypersphere
  Classification}

\begin{icmlauthorlist}
\icmlauthor{Eduard Eiben}{rh}
\icmlauthor{Robert Ganian}{tuw}
\icmlauthor{Iyad Kanj}{chic}
\icmlauthor{Sebastian Ordyniak}{lee}
\icmlauthor{Stefan Szeider}{tuw}
\end{icmlauthorlist}

\icmlaffiliation{tuw}{Algorithms and Complexity Group, TU Wien, Austria.}
\icmlaffiliation{rh}{Royal Holloway, University of London, UK.}
\icmlaffiliation{chic}{DePaul University, USA.}
\icmlaffiliation{lee}{University of Leeds, UK.}

\icmlcorrespondingauthor{Eduard Eiben}{eduard.eiben@rhul.ac.uk}
\icmlcorrespondingauthor{Robert Ganian}{rganian@gmail.com}
\icmlcorrespondingauthor{Iyad Kanj}{ikanj@depaul.edu}
\icmlcorrespondingauthor{Sebastian Ordyniak}{sordyniak@gmail.com}
\icmlcorrespondingauthor{Stefan Szeider}{sz@ac.tuwien.ac.at}

\icmlkeywords{parameterized complexity, classification}

\vskip 0.3in
]

\printAffiliationsAndNotice{\icmlEqualContribution} 

\begin{abstract}
Hypersphere classification is a classical and foundational method that can provide easy-to-process explanations for the classification of real-valued and binary data. However, obtaining an (ideally concise) explanation via hypersphere classification is much more difficult when dealing with binary data
than  real-valued data. 
In this paper, we perform the first complexity-theoretic study of the hypersphere classification problem for binary data. We use the fine-grained parameterized complexity paradigm to analyze the impact of structural properties that may be present in the input data as well as potential conciseness constraints.
Our results include stronger lower bounds and new fixed-parameter
algorithms for hypersphere classification of binary data, which can
find an exact and concise explanation when one exists.
\end{abstract}

\section{Introduction}
With the rapid advancement of Machine Learning (ML) models to automate
decisions, there has been increasing interest in explainable Artificial Intelligence
(XAI), where the ML models can explain their decisions in a way
humans understand. This has led to the reexamination of ML models that are
implicitly easy to explain and interpret with a particular focus on
the conciseness of
explanations~\cite{Doshi-VelezK17,Lipton18,Monroe18c,Ribeiro0G18,ShihCD18,BarceloM0S20,ChalasaniC00J20,DarwicheH20,BlancLT21,Ignatiev0NS21,WaldchenMHK21,IzzaINCM22}.

In this article, we consider a simple
classification task---one of the cornerstones of machine learning---from the viewpoint of XAI. However, unlike previous works on explainability, which have typically targeted questions such as identifying a suitable interpretable model for (area-specific) classification~\cite{NoriCBSK21,ShihTK21,WangZLW21}
or measuring the accuracy cost of
explainability~\cite{LaberM21,MakarychevS21},
the goal of this work is to obtain a comprehensive understanding of the computational complexity of performing binary classification via one of the most fundamental interpretable models.

Consider a set $M$ of either real-valued or binary training feature
points, each represented as a $d$-dimensional feature vector over
$[0,1]$ or $\{0,1\}$ and labeled as either ``blue'' (the set $\VB$) or
``red'' (the set $\VR$). There is, by now, a broad set of more or less
opaque classifiers capable of using such a training set to classify
unlabeled data, where the suitability of each method depends on the
data domain and context; moreover, some classifiers are tailored to
real-valued data, while others are designed for binary (or, more
generally, categorical) data. In this paper, we consider one of the
two arguably simplest---and hence easiest to explain and
visualize---types of classifiers, which can be used in both data
settings: a hypersphere. More formally, the explanations we consider
consist of a cluster center $\vec{c}$ (an element of $[0,1]^d$ or
$\{0,1\}^d$) and distance $\ell$ such that each feature vector is at a
distance at most $\ell$ from $\vec{c}$ if and only if it is blue. 

The reason for studying the complexity of hypersphere classification does not stem purely from the problem's connection to explainability. 
Together with classification by a separating hyperplane,
hypersphere classification represents one of the most classic explanatory examples of classifiers (see~\cite{pcooper,neskovic,neskovic1} to name a few) which have been extensively studied from both the computational geometry and the machine learning perspectives~\cite{astorino1,astorino,pcooper,neskovic,neskovic1,orourke,agarwal,hurtado}. Moreover, hypersphere classification is of special importance in one-class classification due to the inherent asymmetry of the provided explanations~\cite{KimLJ21}. 
While hyperplane separation can be encoded as a linear program and
hence is easily polynomial-time solvable for real-valued and  binary data, the computational complexity of hypersphere classification is far less obvious and has so far remained surprisingly unexplored.

This apparent gap contrasts the situation for several other computational problems arising in the area of machine learning, which have already been targeted by detailed complexity-theoretic studies~\cite{OrdyniakS13,GanianKOS18,SimonovFGP19,DahiyaFPS21,GanianK21,GanianHKOS22,GruttemeierK22}, carried out using the classical as well as the parameterized-complexity paradigms.
In this article, we close the gap by laying bare a detailed map of the problem's computational complexity via the design of novel theoretical algorithms as well as accompanying computational lower bounds.

\smallskip
\noindent
\textbf{Contributions.}\quad 
We begin by observing that the hypersphere classification problem is polynomial-time solvable when the input data is real-valued; in particular, this case can be handled via a more sophisticated linear programming encoding than the one used for the classical hyperplane separation problem. 
However, this approach completely fails when dealing with binary data,
warranting  a more careful complexity-theoretic study of this
case. Our first result shows that hypersphere classification of binary data is not only \NP-hard in general but remains \NP-hard even when there are only two red vectors. We also obtain an analogous hardness result for instances with two blue vectors.

The fact that the problem's complexity differs between real-valued and categorical data is already interesting. However, the \NP-hardness of the latter case does not preclude the existence of efficient algorithms that can solve the problem under additional natural restrictions. Indeed, one of the central themes in modern complexity-theoretic research is the identification of the exact boundaries of tractability. This is frequently achieved through the lens of the \emph{parameterized complexity} paradigm~\cite{DowneyF13,CyganFKLMPPS15}, where we associate each problem instance $\III$ with an integer parameter $k$ (often capturing a certain structural property of the instance) and ask whether the problem of interest can be solved by a ``fixed-parameter'' algorithm, that is, by an algorithm with runtime of the form $f(k)\cdot |\III|^{\bigoh(1)}$ for some computable function $f$. This gives rise to a strong form of computational tractability called \emph{fixed-parameter tractability} (\FPT{}). 

In the case of our problem of interest, it is easy to observe that
hypersphere classification of binary data is fixed-parameter tractable
when parameterized by the data dimension $d$ (since the number of
possible centers is upper-bounded by $2^d$). Moreover, we show that
the problem also admits a fixed-parameter algorithm when parameterized
by the total number of feature points via a combinatorial reduction to
a known tractable fragment of Integer Linear Programming. While these
are important pieces of the complexity-theoretic landscape of
hypersphere classification, these two initial results are somewhat
unsatisfying on their own because (1) they rely on highly restrictive
parameterizations, and (2) they ignore a central aspect of
explainability, which is \emph{conciseness} or
\emph{succinctness}~\cite{Ribeiro0G18,ShihCD18,BlancLT21,WaldchenMHK21,ChalasaniC00J20,IzzaINCM22,OrdyniakPaesaniSzeider23}.

A natural measure of conciseness in our setting is the number of ``1'' coordinates in a vector; indeed, any explanation produced by a classifier will likely  end up ignored by users if such an explanation is incomprehensibly long, relying on too many features.
At the same time, depending on the source of the input data, we may often deal with feature vectors that are already concise. Having concise feature vectors does not necessarily guarantee the existence of a concise center (and vice-versa, concise centers may exist for non-concise data); however, at least one of the two independent measures of conciseness can be expected (or even required) to be small in a variety of settings, making them natural choices for parameters in our analysis.
In the second part of the article, we show that these two conciseness parameters---a bound $\scp$ on the conciseness of the sought-after explanation and a bound $\icon$ on the conciseness of all feature vectors in the training data---can be algorithmically exploited to cope with the \NP-hardness of the hypersphere classification problem for binary data.

\newcommand{\T}[1]{{\small (#1)}}
\begin{table*}[h]
\centering
  \begin{tabular}{@{}c@{\hspace{15mm}}r@{~}l@{\hspace{10mm}}r@{~}l@{\hspace{10mm}}r@{~}l@{\hspace{10mm}}r@{~}l@{}}
    \toprule
    & \multicolumn{8}{c}{\emph{Conciseness}}\\
 \cmidrule{2-9}
 \emph{Structure} & \multicolumn{2}{c}{$\emptyset$~~~~~~~} & \multicolumn{2}{c}{$\scp$~~~~~~~~~~~~~} & \multicolumn{2}{c}{$\icon$~~~~~~~~~~} & \multicolumn{2}{c@{}}{$\scp+\icon$} \\
    \midrule
    $\emptyset$ &
    \NP-h&\T{Thm~\ref{the:2NPh}} &
    \XP~\T{Obs~\ref{obs:econXP}},& \W{2}-h~\T{Thm~\ref{thm:W2_scp}} &
    \NP-h$_{\geq 4}$ & \T{Thm~\ref{thm:rb-np-4}}  &
     \FPT&\T{Thm~\ref{thm:econicon}}  \\ 
     $|\VR|$ & \NP-h$_{\geq 2}$ & \T{Thm~\ref{the:2NPh}} & \XP~\T{Obs~\ref{obs:econXP}},& \W{2}-h~\T{Thm~\ref{thm:W2_scp}}& \FPT&\T{Thm~\ref{thm:colicon}} & \FPT&\T{Thm~\ref{thm:colicon}} \\
		$|\VB|$ & \NP-h$_{\geq 2}$ &
                          \T{Thm~\ref{the:2NPh}} & \XP~\T{Obs~\ref{obs:econXP}},& \W{1}-h~\T{Thm~\ref{thm:W1_scp}}& \FPT& \T{Thm~\ref{thm:colicon}}& \FPT&\T{Thm~\ref{thm:colicon}} \\
		$|\VR\cup \VB|$ & \FPT &
                                  \T{Thm~\ref{thm:fptpoints}} &
                                                                      \FPT
                                 & \T{Thm~\ref{thm:fptpoints}}& \FPT & \T{Thm~\ref{thm:fptpoints}}& \FPT&\T{Thm~\ref{thm:fptpoints}} \\
		$d$ & \FPT&\T{trivial}  & \FPT &\T{trivial} & \FPT &
                                                                     \T{trivial}& \FPT &\T{trivial}\\		
    $\tw$ & \XP	&\T{Cor~\ref{cor:tw}}  & \FPT& \T{Cor~\ref{cor:tw_econ}} & \FPT & \T{Cor~\ref{cor:tw_dcon}}& \FPT&\T{Cor~\ref{cor:tw_econ}} \\
    \bottomrule
	\end{tabular}

	\caption{The complexity landscape of hypersphere classification with respect to combinations of structural and conciseness parameters: $\VR$ and $\VB$ are the sets of red and blue points, respectively; $d$ is the dimension; $\tw$ is the incidence treewidth of the data representation; $\scp$ is the conciseness of the explanation, and $\icon$ is the data conciseness. \NP-h$_{\geq i}$ means that the problem becomes \NP-hard for parameter values of at least $4$, while $\W{j}$-h means that the problem is hard for the complexity class $\W{j}$ and hence is unlikely to be fixed-parameter tractable~\cite{DowneyF13}. 
	\label{tab:results}\vspace{-0.5cm}}
\end{table*}

Toward understanding the complexity of hypersphere classification of binary data  through the perspective of conciseness constraints, we begin by considering restrictions on the data conciseness $\icon$. We obtain a tight classification by showing that the problem is polynomial-time tractable when $\icon\leq 3$ via a reduction to a tractable fragment of the constraint satisfaction problem and \NP-hard otherwise. Moreover, we obtain fixed-parameter algorithms parameterized by $\icon$ plus  the number of red or blue points, circumventing the earlier \NP-hardness results.
When considering the explanation conciseness, we show that hypersphere classification is \XP-tractable when parameterized by $\scp$ and at the same time provide evidence excluding fixed-parameter tractability even parameterized by $\scp$ together with the number of red or blue points.
Finally, we obtain a linear-time fixed-parameter algorithm for the problem parameterized by $\scp+\icon$.

While this settles the complexity of binary-data hypersphere
classification from the perspective of conciseness measures, the
obtained lower bounds imply that neither measure of conciseness (i.e.,
neither $\scp$ nor $\icon$) suffices to achieve fixed-parameter
tractability on its own.
As our final contribution, we consider whether achieving tractability for the problem is possible by exploiting a suitable structural measure of the input data. In particular, following recent successes in closely related areas such as clustering and data completion~\cite{GanianHKOS22,GanianKOS18}, we consider the \emph{incidence treewidth} $\tw$ of the data representation. Using a non-trivial dynamic programming procedure, we obtain a fixed-parameter algorithm for binary-data hypersphere classification parameterized by either $\tw+\scp$ or $\tw+\icon$; in other words, each of the two notions of conciseness suffices for tractability of hypersphere clustering for data that is ``well-structured,'' in the sense of having small incidence treewidth. Moreover, as a byproduct of our algorithm, we also obtain the \XP-tractability of binary-data hypersphere classification parameterized by $\tw$ alone.

A summary of our results is provided in Table~\ref{tab:results}.

\smallskip
\noindent \textbf{Related Work.}\quad

While there is, to the best of our knowledge, no prior work targeting the complexity of hypersphere classification of binary data, there is a significant work on the real-valued variant of the problem by the machine learning community~\cite{pcooper,neskovic,neskovic1,astorino,astorino1}, where they studied the optimization version of the problem in which one seeks the smallest bounding sphere that separates the blue points from the red ones. Our results extend to the optimization version of the problem for binary data, as mentioned in Section~\ref{sec:conclusion}. Many of the above works consider relaxations of the real-valued optimization problem, in which the sphere sought is not of minimum radius~\cite{neskovic,neskovic1,astorino,astorino1}---allowing for error or for outliers---and reduce the problem to some fragment of quadratic programming. We point out that the problem is also related to that of finding a minimum bounding sphere to distinguish/discriminate a set of objects (i.e., one-class classification)~\cite{tax}, which is, in turn, inspired by the Support Vector Machine models introduced in~\cite{vapnik}.  

The hypersphere classification of real-valued low-dimensional data has also been studied in the context of point separability within the field of computational geometry. In particular, the separability of two sets of points in $\mathbb{R}^2$ by a circle was studied by O'Rourke, Kosaraju and Megiddo (\citeyear{orourke}), who established the linear-time tractability of that case. They also observed that their result could be lifted to an $\bigoh(n^{d})$-time algorithm for the separability of $n$ $d$-dimensional data points by a hypersphere; however this is superseded by the $n^{\bigoh(1)}$-time algorithm observed in Proposition~\ref{obs:propreal}, which runs in polynomial time even for unbounded values of $d$. Several authors also studied related point-separation problems in $\mathbb{R}^2$ and $\mathbb{R}^3$, such as separability of points via polyhedra~\cite{megiddo1988complexity}, 
L-shapes~\cite{SheikhiMBD15}
and a variety of other objects~\cite{agarwal,alegria2022separating}.

\section{Preliminaries}
\label{sec:prelim}
For $\ell \in \Nat$, we write $[\ell]$ for $\{1, \ldots, \ell\}$. 
For convenience, we identify each vector $\vec{v}=(v_1,\dots,v_d)$ with the point $(v_1,\dots,v_d)$ in $d$-dimensional space.

 \paragraph{Problem Definition and Terminology.}
 For two vectors $\vec{a}, \vec{b} \in \{0,1\}^d$, we denote by 
$\delta(\vec{a}, \vec{b})$ the Hamming distance between $\vec{a}$ and $\vec{b}$. For a vector $\vec{v} \in \{0,1\}^d$ and $r \in \Nat$, denote by $B(\vec{v}, r)$ the hypersphere (i.e., ball) centered at $\vec{v}$ and of radius $r$; that is, the set of all vectors $\vec{x} \in \{0,1\}^d$ satisfying $\delta(\vec{v}, \vec{x}) \leq r$. Similarly, for vectors over $[0,1]^d$ we denote by $B(\vec{v}, r)$ the hypersphere (i.e., ball) centered at $\vec{v}$ and of radius $r$ with respect to the Euclidean distance in $\mathbb{R}^d$.

The problem under consideration in this paper is defined as follows:

\newcommand{\RB}{\textsc{BHC}}
\pbDef{\textsc{Binary Hypersphere Classification} (\RB)}{A set $V=\VR \cup \VB$ of $d$-dimensional vectors over the binary domain $D=\{0, 1\}$, where $\VR \cap \VB = \emptyset$.}{Is there a vector 
$\vec{c} \in D^{d}$ and $r \in \Nat$ such that $\VB \subseteq B(\vec{c}, r)$ and $\VR \cap B(\vec{c}, r) = \emptyset$?}

Throughout the paper, we will refer to the vectors in $\VR$ as ``\red'' and those in $\VB$ as ``\blue''. We also denote by $\vO(\vec{v})$ the set
of all coordinates $i$ such that $\vec{v}[i]=1$ and we write
$\con(\vec{v})$ for the \emph{conciseness} of $\vec{v}$, i.e.,
$\con(\vec{v})=|\vO(\vec{v})|$.
Moreover, observe that the Hamming distance $\delta(\vec{v},\vec{c})$ of two vectors $\vec{v}$ and
$\vec{c}$ can be written as
$\delta(\vec{v},\vec{c})=|\vO(\vec{c})|+|\vO(\vec{v})|-2|\vO(\vec{v})\cap
\vO(\vec{c})|$. 

Naturally, one may also consider the analogous problem of \textsc{Real-valued Hypersphere Classification}, where the only distinction is that the domain is $[0,1]$ instead of $\{0,1\}$. As mentioned in the introduction, this problem can be shown to be polynomial-time solvable and hence is not considered further in our complexity-theoretic analysis.

\begin{proposition}
\label{obs:propreal}
\textsc{Real-valued Hypersphere Classification} can be solved in polynomial time.
\end{proposition}

\iflong
\begin{proof}
We give a reduction to \textsc{Linear Programming}, i.e., the task of finding a solution to a set of linear inequalities. We construct a variable $x_i$ for each of the dimensions $i\in [d]$ which will capture the coordinates of a hypothetical center $\vec{c}$. For each pair of vectors $\vec{r}\in \VR$ and $\vec{b}\in \VB$, we introduce a constraint which ensures that $c$ will be closer to the latter than the former. Note that every $\vec{c}$ satisfying this property (for all pairs of blue and red vectors) can serve as a center for a hypersphere containing only blue vectors---one can simply set the radius to the maximum distance of a blue vector from $\vec{c}$.

To define this constraint, we observe that the distance between $\vec{c}$ and $\vec{r}$ is precisely equal to $\sqrt{\sum_{i\in [d]}(\vec{r}[i]-x_i)^2}$, and similarly for $\vec{b}$. Hence, we need a constraint ensuring that 

$$\sqrt{\sum_{i\in [d]}(\vec{r}[i]-x_i)^2}> \sqrt{\sum_{i\in [d]}(\vec{b}[i]-x_i)^2},$$

which can be equivalently reduced to

$$\sum_{i\in [d]}(\vec{r}[i]-x_i)^2> \sum_{i\in [d]}(\vec{b}[i]-x_i)^2.$$

Since the quadratic terms of the $x_i$ variable are the same on both sides, the inequality can be simplified by removing them completely.

$$\sum_{i\in [d]}\vec{r}[i]^2-\vec{b}[i]^2 > \sum_{i\in [d]}2(\vec{r}[i]-\vec{b}[i])x_i.$$

Hence, each constraint can be stipulated by adding a linear inequality into the linear program. The claim then follows by the well-known polynomial-time tractability of \textsc{Linear Programming}.
\end{proof}
\fi

\ifshort
\paragraph{Parameterized Complexity.} 
In parameterized
algorithmics~\cite{FlumGrohe06,DowneyF13,CyganFKLMPPS15} the
running-time of an algorithm is studied with respect to a parameter
$k\in\Nat_0$ and input size~$n$. The basic idea is to find a parameter
that describes the structure of the instance such that the
combinatorial explosion can be confined to this parameter. In this
respect, the most favorable complexity class is \FPT
(\textit{fixed-parameter tractable}) which contains all problems that
can be decided by an algorithm running in time $f(k)\cdot
n^{\bigoh(1)}$, where $f$ is a computable function. Algorithms with
this running-time are called \emph{fixed-parameter algorithms}. A less
favorable outcome is an \XP{} \emph{algorithm}, which is an algorithm
running in time $\bigoh(n^{f(k)})$; problems admitting such
algorithms belong to the class \XP. 

Showing that a parameterized problem is hard for the complexity classes \W{1} or \W{2} rules out the existence of a fixed-parameter algorithm under well-established complexity-theoretic assumptions. Such hardness results are typically established via a \emph{parameterized reduction}, which is an analogue of a classical polynomial-time reduction with two notable distinctions: a parameterized reduction can run in time $f(k)\cdot
n^{\bigoh(1)}$, but the parameter of the produced instance must be upper-bounded by a function of the parameter in the original instance.
\fi

\iflong
\paragraph{Parameterized Complexity.} In parameterized complexity~\cite{FlumGrohe06,DowneyF13,CyganFKLMPPS15},
the complexity of a problem is studied not only with respect to the
input size, but also with respect to some problem parameter(s). The
core idea behind parameterized complexity is that the combinatorial
explosion resulting from the \NP-hardness of a problem can sometimes
be confined to certain structural parameters that are small in
practical settings. We now proceed to the formal definitions.

A {\it parameterized problem} $Q$ is a subset of $\Omega^* \times
\mathbb{N}$, where $\Omega$ is a fixed alphabet. Each instance of $Q$ is a pair $(I, \kappa)$, where $\kappa \in \Nat$ is called the {\it
parameter}. A parameterized problem $Q$ is
{\it fixed-parameter tractable} (\FPT)~\cite{FlumGrohe06,DowneyF13,CyganFKLMPPS15}, if there is an
algorithm, called an {\em \FPT-algorithm},  that decides whether an input $(I, \kappa)$
is a member of $Q$ in time $f(\kappa) \cdot |I|^{\bigoh(1)}$, where $f$ is a computable function and $|I|$ is the input instance size.  The class \FPT{} denotes the class of all fixed-parameter
tractable parameterized problems.

A parameterized problem $Q$
is {\it \FPT-reducible} to a parameterized problem $Q'$ if there is
an algorithm, called an \emph{\FPT-reduction}, that transforms each instance $(I, \kappa)$ of $Q$
into an instance $(I', \kappa')$ of
$Q'$ in time $f(\kappa)\cdot |I|^{\bigoh(1)}$, such that $\kappa' \leq g(\kappa)$ and $(I, \kappa) \in Q$ if and
only if $(I', \kappa') \in Q'$, where $f$ and $g$ are computable
functions. By \emph{\FPT-time}, we denote time of the form $f(\kappa)\cdot |I|^{\bigoh(1)}$, where $f$ is a computable function.
Based on the notion of \FPT-reducibility, a hierarchy of
parameterized complexity, {\it the \cc{W}-hierarchy} $=\bigcup_{t
\geq 0} \W{t}$, where $\W{t} \subseteq \W{t+1}$ for all $t \geq 0$, has 
been introduced, in which the $0$-th level \W{0} is the class {\it
\FPT}. The notions of hardness and completeness have been defined for each level
\W{$i$} of the \cc{W}-hierarchy for $i \geq 1$ \cite{DowneyF13,CyganFKLMPPS15}. It is commonly believed that $\W{1} \neq \FPT$ (see \cite{DowneyF13,CyganFKLMPPS15}). The
\W{1}-hardness has served as the main working hypothesis of fixed-parameter
intractability. The class \XP{} contains parameterized problems that can be solved in time  $\bigoh(|I|^{f(\kappa)})$, where $f$ is a computable function; it
contains the class \W{t}, for $t \geq 0$, and every problem in \XP{} is polynomial-time solvable when the parameters are bounded by a constant.
The class \paraNP{} is the class of parameterized problems that can be solved by non-deterministic algorithms in time $f(\kappa)\cdot |I|^{\bigoh(1)}$, where $f$ is a computable function.
A problem is \emph{\paraNP{}-hard} if it is \NP-hard for a constant value of the parameter~\cite{FlumGrohe06}.
\fi

\paragraph{Treewidth.}
A \emph{nice tree-decomposition}~$\mathcal{T}$ of a graph $G=(V,E)$ is a pair 
$(T,\chi)$, where $T$ is a tree (whose vertices are called \emph{nodes}) rooted at a node $t_r$ and $\chi$ is a function that assigns each node $t$ a set $\chi(t) \subseteq V$ such that the following hold: 
\begin{itemize}[noitemsep,topsep=0pt]
	\item For every $uv \in E$ there is a node
	$t$ such that $u,v\in \chi(t)$.
	\item For every vertex $v \in V$,
	the set of nodes $t$ satisfying $v\in \chi(t)$ forms a subtree of~$T$.
	\item $|\chi(\ell)|=0$ for every leaf $\ell$ of $T$ and $|\chi(t_r)|=0$.
	\item There are only three kinds of non-leaf nodes in $T$:
	\begin{itemize}[noitemsep]
        \item \textbf{Introduce node:} a node $t$ with exactly
          one child $t'$ such that $\chi(t)=\chi(t')\cup
          \{v\}$ for some vertex $v\not\in \chi(t')$.
        \item \textbf{Forget node:} a node $t$ with exactly
          one child $t'$ such that $\chi(t)=\chi(t')\setminus
          \{v\}$ for some vertex $v\in \chi(t')$.
        \item \textbf{Join node:} a node $t$ with two children $t_1$,
          $t_2$ such that $\chi(t)=\chi(t_1)=\chi(t_2)$.
	\end{itemize}
\end{itemize}

The \emph{width} of a nice tree-decomposition $(T,\chi)$ is the size of a largest set $\chi(t)$ minus~$1$, and the \emph{treewidth} of the graph $G$,
denoted $\tw(G)$, is the minimum width of a nice tree-decomposition of~$G$.

We let $T_t$ denote the subtree of $T$ rooted at a node $t$, and use $\chi(T_t)$ to denote the set $\bigcup_{t'\in V(T_t)}\chi(t')$ and \(G_t\) to denote the graph \(G[\chi(T_t)]\) induced by the vertices in \(\chi(T_t)\).
Efficient fixed-parameter algorithms are known for computing a nice tree-decomposition of near-optimal width:

\begin{proposition}[\citealt{Kloks94,Korhonen21}]\label{fact:findtw}%
	There exists an algorithm which, given an $n$-vertex graph $G$ and an integer~$k$, in time $2^{\bigoh(k)}\cdot n$ either outputs a nice tree-decomposition of $G$ of width at most $2k+1$ and $\bigoh(n)$ nodes, or determines that $\tw(G)>k$.
\end{proposition}

\paragraph{Constraint Satisfaction Problems.} Let $D=\{0,1\}$ and let $n$ an integer.
An $n$-ary relation  on $D$ is a subset of $D^n$.
An instance $I$ of a \emph{Boolean constraint satisfaction problem} (CSP) 
is a pair $(V,C)$, where $V$ is a finite set of variables
and $C$ is a set of constraints. A
\emph{constraint} $c \in C$ consists of a \emph{scope}, denoted by
$V(c)$, which is an ordered list of a subset of~$V$, and a relation,
denoted by $R(c)$, which is a $|V(c)|$-ary relation on $D$; $|V(c)|$
is the \emph{arity} of~$c$. 
\iflong To simplify notation, we sometimes treat
ordered lists without repetitions, such as the scope of a constraint,
like sets.  For a variable $v \in V(c)$ and a tuple $t \in R(c)$, we
denote by $t[v]$, the $i$-th entry of $t$, where $i$ is the position
of $v$ in $V(c)$.  
\fi
 
A \emph{solution} to a CSP instance $I=(V,C)$ is a mapping $\tau : V
\rightarrow D$
such that $\tuple{\tau(v_1),\dotsc,\tau(v_{|V(c)|})}
\in R(c)$ for every $c \in C$ with $V(c)=\tuple{v_1,\dotsc,v_{|V(c)|}}$.
A CSP instance is \emph{satisfiable} if and only if it has at least one solution.

\iflong
A \emph{constraint language} $\Gamma$ over $D$ is a set of
relations over $D$. $\CCC_\Gamma$ denotes the
class of CSP instances $I$ with the property that for each $c\in C(I)$
we have $R(c)\in \Gamma$. $\CSP(\Gamma)$ refers to the CSP with
instances restricted to~$\CCC_\Gamma$.  A constraint language $\Gamma$
is \emph{tractable} if $\CSP(\Gamma)$ can be solved in polynomial
time.

Given a $k$-ary relation $R$ over $D$ and a function
$\phi:D^n \rightarrow D$, we say that $R$ is {\em closed under
  $\phi$}, if for all collections of $n$ tuples $t_1,\dotsc,t_n$ from
$R$, the tuple $\tuple{\phi(t_1[1],\dotsc,t_n[1]),
  \dotsc,\phi(t_1[k],\dotsc,t_n[k])}$ belongs to $R$. The function
$\phi$ is also said to be a {\em polymorphism of $R$}.
We say that a constraint language $\Gamma$ is \emph{closed under} $\phi$ (or
equivalently that $\phi$ is a \emph{polymorphism} for $\Gamma$) if all
relations in $R$ are closed under $\phi$. Similarily, we say that a
CSP instance $I=(V,C)$ is \emph{closed under} $\phi$ if $R(c)$ is
closed under $\phi$ for every constraint $c \in C$.

We need the following well-known (types) of operations:
\begin{itemize}
\item An operation $\phi : D \rightarrow D$ is \emph{constant} if
  there is a $d \in D$ such that for every $d' \in D$, it
  holds that $\phi(d')=d$.
\item An operation $\phi : D^2 \rightarrow D$ is a
  \emph{AND}/\emph{OR} operation if there is an ordering of the
  elements of $D$ such that for every $d,d' \in D$, it holds that
  $\phi(d,d')=\phi(d',d)=\min\{d,d'\}$ or
  $\phi(d,d')=\phi(d',d)=\max\{d,d'\}$, respectively.
\item An operation $\phi : D^3 \rightarrow D$ is a \emph{majority}
  operation if for every $d,d' \in D$ it holds that
  $\phi(d,d,d')=\phi(d,d',d)=\phi(d',d,d)=d$.
\item An operation $\phi : D^3 \rightarrow D$ is an \emph{minority}
  operation if for every $d,d' \in D$ it holds
  that $\phi(d,d,d')=\phi(d,d',d)=\phi(d',d,d)=d'$.
\end{itemize}
The following is commonly known as Schaefer's
dichotomy theorem:
\begin{theorem}[\citealt{Schaefer78,Chen09}]\label{thm:schaefer}
  Let $\Gamma$ be a finite constraint language over $D$. Then,
  CSP$(\Gamma)$ is tractable in polynomial-time if $\Gamma$ is closed
  under any of the following (types) of operations: (1) a constant
  operation, (2) an AND or an OR operation, (3) a majority operation,
  or (4) a minority operation. Otherwise, CSP$(\Gamma)$ is \NP-complete.
\end{theorem}

We will later use reductions to Boolean CSP instances that are closed
under a majority opertion. It is easy to see that there is only one
majority operation defined on the Boolean domain, which we will refer
to from here onwards as the Boolean majority operation.

The following proposition provides the exact run-time to solve any Boolean CSP instance that is
closed under the unique Boolean majority operation. The proposition follows because as observed in~\cite{jcg97} any
such CSP is equivalent to an instance of 2-Satisfiability, i.e.,
deciding the satisfiability of a propositional formula in 2-CNF,
together with the fact that 2-Satisfiability can be solved in
linear-time~\cite{AspvallPT79}.
\begin{proposition}[\citealt{jcg97,AspvallPT79}]\label{pro:solve-csp}
  Any Boolean CSP instance $I=(V,C)$ that is closed under the unique
  Boolean majority operation can be solved in time $\bigoh(|V|+|C|)$.
\end{proposition}

Next we will introduce certain relations and show that they are closed under the unique
majority operation for Boolean CSPs.
Let $R^a_{\leq x}$ and $R^a_{\geq x}$ be the Boolean $a$-ary relations
containing all tuples that contain at most $x$ $1$'s or at least $x$
$1$'s, respectively. Moreover, let $R_O^a$ and $R_Z^a$ be the $a$-ary Boolean
relations that contain only the all-one and all-zero tuple, respectively.
\begin{lemma}\label{lem:closed-maj}
  Let $a$ be an integer. Then, $R_O^a$ and $R_Z^a$ are closed
  under the Boolean majority operation.
  Moreover, if $1 \leq a \leq 3$, then $R^a_{\leq 1}$ and $R^a_{\geq
    2}$ are closed under the Boolean majority operation. 
\end{lemma}
\begin{proof}
  The claim of the lemma clearly holds for $R_O^a$ and $R_Z^a$ since
  each contains only one tuple.
  
  Thus, let $t_1$, $t_2$, and $t_3$ be three distinct tuples in $R^a_{\leq 1}$;
  note that if the tuples are not distinct then the majority
  of the three tuples is the one that occurs more than once and is
  therefore trivially part of $R^a_{\leq 1}$. Since each
  of the three tuples contains at most one $1$ and since all
  tuples are distinct, the majority of the three tuples will always be
  equal to the all-zero tuple, which is clearly part of $R^a_{\leq 1}$.

  Finally, the claim of the lemma for $R^a_{\geq 2}$ follows since
  $R_{\geq 2}^a$ is the complement of $R_{\leq 1}^a$.
\end{proof}
\fi
 
\section{\NP-hardness of \RB\ Restricted to Two Red or Blue Vectors}
\label{sec:nphardness2}

In this section, we show that \RB\ remains \NP-hard even when one of the two sets ($\VB$ or $\VR$) has size at most two. In particular, let us denote by \textsc{2Red-}\RB{} the restriction of \RB{} to instances in which the number of \red{} vectors is two (i.e., $|\VR|=2$), and by \textsc{2Blue-}\RB{} the restriction of \RB{} to instances in which the number of \blue{} vectors is two (i.e., $|\VB|=2$). 

\begin{theorem}
\label{the:2NPh}
\textsc{2Red-}\RB{} and \textsc{2Blue-}\RB{} are \NP-complete.
\end{theorem}

\ifshort
\begin{proof}[Proof Sketch]
Proving membership in \NP{} is straightforward and is omitted. 
We begin with the following problem, which is known to be \NP-hard~\cite{litman}:

\pbDef{Minimum Radius (MR){}}{A set $V$ of $(2n)$-dimensional binary vectors where $n \in \Nat$.}{Is there a $(2n)$-dimensional center vector 
$\vec{c}$ such that $V \subseteq B(\vec{c}, n)$?}

Denote by Rest-MR the restriction of MR to instances in which $V$ contains the $(2n)$-dimensional all-zero vector $\vec{0}_{2n}$ and we ask for a center vector $\vec{c}$ that contains exactly $n$ ones. We first show that Rest-MR remains \NP-hard via a polynomial-time Turing-reduction from MR to Rest-MR. 

At this point, to complete the proof of the theorem it suffices to exhibit a polynomial-time reduction from Rest-MR to \textsc{2Red-}\RB{} (an analogous reduction is also be used for \textsc{2Blue-}\RB{}). Given an instance $V$ of Rest-MR, we construct an instance $V^+$ of 2Red-\RB{} such that $V$ is a \yes-instance of Rest-MR if and only if $V^+$ is a \yes-instance of \textsc{2Red-}\RB{}. Without loss of generality, we may assume that (the $(2n)$-dimensional all 1's vector) $\vec{1}_{2n} \in V$ since $V$ is a \yes-instance of Rest-MR if and only if $V \cup \{\vec{1}_{2n} \}$ is. The previous statement is true since $\vec{0}_{2n} \in V \subseteq B(\vec{c}, n)$, where $\vec{c}$ contains exactly $n$ 1's, if and only if $(V \cup \{\vec{1}_{2n}\}) \subseteq B(\vec{c}, n)$.

To construct $V^+$, we extend each $(2n)$-dimensional vector $\vec{v} \in V$ by adding two coordinates, that we refer to as coordinates $q_{2n+1}$ and $q_{2n+2}$, and setting their values to 0 and 1, respectively; let $\vec{v}^{+}$ denote the extension of $\vec{v}$. Let $V_b$ be the resulting set of (extended) vectors from $V$, and let $V_r=\{\vec{0}_{2n+2}, \vec{1}_{2n+2}\}$, 
where $\vec{0}_{2n+2}, \vec{1}_{2n+2}$ are the $(2n+2)$-dimensional all-zero and all-one vectors, respectively. Finally, let $V^+= V_b \cup V_r$. 
We can now show that $V$ is a \yes-instance of Rest-MR if and only if $V^+$ is a \yes-instance of 2Red-\RB. 
\end{proof}
\fi

\iflong
\begin{proof}
Proving membership in \NP{} is straightforward and is omitted. We will first show the \NP-hardness of \textsc{2Red-}\RB{} and explain how the proof can be slightly modified to yield the \NP-hardness of 2Blue-\RB.  

The following problem is known to be \NP-hard~\cite{litman}:

\pbDef{Minimum Radius (MR){}}{A set $V$ of $(2n)$-dimensional binary vectors where $n \in \Nat$.}{Is there a $(2n)$-dimensional center vector 
$\vec{c}$ such that $V \subseteq B(\vec{c}, n)$?}

Denote by Rest-MR the restriction of MR to instances in which $V$ contains the $(2n)$-dimensional all-zero vector $\vec{0}_{2n}$ and we ask for a center vector $\vec{c}$ that contains exactly $n$ ones. We first show that Rest-MR remains \NP-hard. To show that, we exhibit a polynomial-time Turing-reduction from MR to Rest-MR. 

For each $\vec{x} \in V$, define the set of vectors $V_{\vec{x}}$ obtained from $V$ by normalizing the vectors in $V$ so that $\vec{x}$ is the $\vec{0}_{2n}$ vector; that is, for each vector $\vec{y} \in V$, $V_{\vec{x}}$ contains the vector $\vec{y'}$ where 
the $i$-th coordinate of $\vec{y'}$ is 0 if and only if the $i$-th coordinate of $\vec{y}$ is equal to the $i$-th coordinate of $\vec{x}$. Observe that the map $\Pi_{\vec{x}}: \{0, 1\} ^{2n} \longrightarrow \{0, 1\}^{2n}$, where $\Pi_{\vec{x}}(\vec{y}) = \vec{y'}$ (as described above) is a bijection. Moreover, it preserves the Hamming distance: For any two vectors $\vec{u}, \vec{w} \in \{0, 1\}^{2n}$, we have $\delta(\vec{u}, \vec{w}) = \delta(\Pi_{\vec{x}}(\vec{u}),\Pi_{\vec{x}}(\vec{w}))$.  Clearly, each $V_{\vec{x}}$  is a set of $(2n)$-dimensional vectors that contains $\vec{0}_{2n}$, which is the vector $\Pi_{\vec{x}}(\vec{x})$.  
We claim that $V$ is a \yes-instance of MR if and only if there exists $\vec{x} \in V$ such that $V_{\vec{x}}$ is a \yes-instance of Rest-MR. This would show that MR is polynomial-time Turing reducible to Rest-MR.

One direction is easy to see. If $V_{\vec{x}}$ is a \yes-instance of Rest-MR for some $\vec{x} \in V$, then there is a vector $\vec{c'}$ containing exactly $n$ 1's such that $V_{\vec{x}} \subseteq B(\vec{c'}, n)$. By the properties of the mapping $\Pi_{\vec{x}}$, it is easy to see that 
 $V \subseteq B(\vec{c}, n)$, where $\vec{c} = \Pi_{\vec{x}}^{-1}(\vec{c'})$, and hence, $V$ is a \yes-instance of MR. To prove the converse, suppose that $V$ is \yes-instance of MR. Then there exists $\vec{c} \in \{0, 1\}^{2n}$ such that $V \subseteq B(\vec{c}, n)$. Let $\vec{x}$ be the vector in $V$ that is farthest away from $\vec{c}$; that is, $\vec{x}$ is a vector in $V$ with the maximum Hamming distance to $\vec{c}$. Let $V_{\vec{x}} = \{\Pi_{\vec{x}}(\vec{y}) \mid \vec{y} \in V\}$, and let $\vec{c'}=\Pi_{\vec{x}}(\vec{c})$. By the properties of $\Pi_{\vec{x}}$, it holds that $V_{\vec{x}} \subseteq B(\vec{c'}, n)$. Since $\vec{0}_{2n}= \Pi_{\vec{x}}(\vec{x}) \in V_{\vec{x}}$, it follows that $\vec{c'}$ contains at most $n$ ones. Moreover, since $\vec{x}$ is a vector in $V$ that is farthest away from $\vec{c}$, $\vec{0}_{2n}$ is a vector in $V_{\vec{x}}$ that is farthest away from $\vec{c'}$. If $\vec{c'}$ contains fewer than $n$ 1's, let $r$ be the number of 1's in $\vec{c'}$. Observe that $\delta(\vec{c'}, \vec{0}_{2n}) =r$ and that $r$ is the maximum distance between $\vec{c'}$ and any vector in $V_{\vec{x}}$. By flipping any $n-r$ 0's in $\vec{c'}$ (to 1's), we obtain a vector $\vec{c''}$ that contains exactly $n$ 1's and satisfying $V_{\vec{x}} \subseteq B(\vec{c''}, n)$. This shows that $V_{\vec{x}}$ is a \yes-instance of Rest-MR and completes the proof of the claim.
  
Therefore, to complete the proof of the theorem, it suffices to exhibit a polynomial-time (many-one) reduction from Rest-MR to \textsc{2Red-}\RB{}. Let $V$ be an instance of Rest-MR. We will construct in polynomial time an instance $V^+$ of 2Red-RB{} such that $V$ is a \yes-instance of Rest-MR if and only if $V^+$ is a \yes-instance of \textsc{2Red-}\RB{}. Without loss of generality, we may assume that (the $(2n)$-dimensional all 1's vector) $\vec{1}_{2n} \in V$ since $V$ is a \yes-instance of Rest-MR if and only if $V \cup \{\vec{1}_{2n} \}$ is. The previous statement is true since $\vec{0}_{2n} \in V \subseteq B(\vec{c}, n)$, where $\vec{c}$ contains exactly $n$ 1's, if and only if $(V \cup \{\vec{1}_{2n}\}) \subseteq B(\vec{c}, n)$.

To construct $V^+$, we extend each $(2n)$-dimensional vector $\vec{v} \in V$ by adding two coordinates, that we refer to as coordinates $q_{2n+1}$ and $q_{2n+2}$, and setting their values to 0 and 1, respectively; let $\vec{v}^{+}$ denote the extension of $\vec{v}$. (It does not really matter which coordinate is set to 0 and which is set to 1 provided that it is done consistently over all the vectors.) Let $V_b$ be the resulting set of (extended) vectors from $V$, and let $V_r=\{\vec{0}_{2n+2}, \vec{1}_{2n+2}\}$, 
where $\vec{0}_{2n+2}, \vec{1}_{2n+2}$ are the $(2n+2)$-dimensional all-zero and all-one vectors, respectively. Finally, let $V^+= V_b \cup V_r$. We claim that $V$ is a \yes-instance of Rest-MR if and only if $V^+$ is a \yes-instance of 2Red-\RB. 

In effect, suppose that $V$ is a \yes-instance of Rest-MR. Then there exists a $(2n)$-dimensional vector $\vec{c}$ containing exactly $n$ 1's such that $V \subseteq B(\vec{c}, n)$. Consider the $(2n+2)$-dimensional extension vector $\vec{c}^{+}$ of $\vec{c}$, whose $q_{2n+1}$ coordinate is 0 and $q_{2n+2}$ coordinate is 1. Observe that, for every $\vec{v} \in V$, we have $\delta(\vec{c}, \vec{v}) = \delta(\vec{c}^+, \vec{v}^+)=n$. It follows that $V_b \subseteq B(\vec{c}^+, n)$. Moreover, we have
 $\delta(\vec{c}^+, \vec{0}_{2n+2})=\delta(\vec{c}^+, \vec{1}_{2n+2})=n+1$, and hence, $V_r \cap B(\vec{c}^+, n)=\emptyset$. It follows that $V^+$ is a \yes-instance of 2Red-\RB. To prove the converse, suppose that $V^+$ is a \yes-instance of 2Red-\RB. Then there exists $\vec{c}^+ \in \{0,1\}^{2n+2}$ and $r \in \Nat$ such that $V_b \subseteq B(\vec{c}^+, r)$ and $V_r \cap B(\vec{c}^+, r)=\emptyset$. Since $\vec{0}_{2n}^{+}$ is in $B(\vec{c}^+, r)$ and $\vec{0}_{2n+2}$ is not, it follows that coordinate $q_{2n+2}$ of $\vec{c}^+$ is 1. Similarly, since $\vec{1}_{2n}^{+}$ is in  $B(\vec{c}^+, r)$ and $\vec{1}_{2n+2}$ is not, it follows that coordinate $q_{2n+1}$ of $\vec{c}^+$ is 0.
Let $\vec{c}$ be the restriction of $\vec{c}^+$ to the first $2n$ coordinates. If $\vec{c}$ contains fewer than $n$ 1's, then its distance from $\vec{1}_{2n}^{+}$ would not be less than its distance from $\vec{0}_{2n+2}$, contradicting the fact that 
$\vec{1}_{2n}^{+} \in V_b$ and $\vec{0}_{2n+2} \in V_r$. The argument is similar if  $\vec{c}$ contains more than $n$ 1's. It follows from the above that  $\vec{c}$ contains exactly $n$ 1's and that $r=n$. Furthermore, since coordinates $q_{2n+1}$ and $q_{2n+2}$ of  $\vec{c}^+$ are $0$ and $1$, respectively, which match the values of these coordinates for each $\vec{v}^+$, where $\vec{v} \in V$, it follows that, for every $\vec{v} \in V$, we have $\delta(\vec{c}, \vec{v}) = \delta(\vec{c}^+, \vec{v}^+)$. Since $V_b \subseteq B(\vec{c}^+, r=n)$, it follows that $V \subseteq B(\vec{c}, n)$, where $\vec{c}$ contains exactly $n$ 1's. This shows that $V$ is a \yes-instance of Rest-MR and completes the proof.

Finally, to prove the \NP-hardness of 2Blue-\RB, we make the following observation. Denote by $\overline{\vec{c}^+}$ the bit-wise \emph{complement} vector of $\vec{c}^+$ (i.e., the vector obtained from $\vec{c}^+$ by flipping each coordinate of $\vec{c}^+$).  It is straightforward to see that $V_b \subseteq B(\vec{c}^+, n)$ and $V_r \cap  B(\vec{c}^+, n) = \emptyset$ if and only if $V_r \subseteq B(\overline{\vec{c}^+}, n+1)$ and $V_b \cap  B(\vec{c}^+, n+1) =\emptyset$. Therefore, by proceeding as in the reduction above but switching the sets $V_r$ and $V_b$ at the end, we obtain a polynomial-time Turing reduction from MR to  \textsc{2Blue-}\RB{} thus showing the \NP-hardness of  2Blue-\RB.  
\end{proof}
\fi

\section{Basic Parameterizations for \RB}
\label{sec:fptpoints}

We follow up on Theorem~\ref{the:2NPh} by considering the two remaining obvious parameterizations of the problem, notably $d$ and $|V|$. The former case is trivial since it bounds the size of the input.

\begin{observation}
\label{obs:d}
\RB{} is \FPT\ parameterized by $d$.
\end{observation}

Next, we give a fixed-parameter algorithm for \RB{} parameterized by the total number of vectors.

\begin{theorem}
\label{thm:fptpoints}
\RB{} is \FPT{} parameterized by the number of \red{} vectors plus \blue{} vectors.
\end{theorem}

\ifshort
\begin{proof}[Proof Sketch]
Let $k$ be the total number of \red{} vectors plus \blue{} vectors. 
For convenience, we will consider the matrix representation of the input, in which the vectors are represented as the rows of a matrix $M$. 
Observe that, since there are $k$ rows of binary coordinates in $M$, the total number of different column configurations of $M$ is at most $2^k$. 
The idea behind the fixed-parameter algorithm is to encode the problem as an instance of an Integer Linear Program (ILP) with $2^{k+1}$ variables---two for every column type. One such variable will capture the number of ``$1$''s the center uses in the columns belonging to that type, while the other simply captures the number of ``$0$''s of the center in these columns. The main reason why this suffices is that the exact positions of these ``$0$''s and ``$1$''s within these column types is irrelevant when considering the distance between an arbitrary point and a center. The constraints simply ensure that the distance between the center and each blue point is strictly smaller than the distance between the center and each red point.

It is well known that such an ILP instance can then be solved in \FPT{} time using the classical result of Lenstra~\cite{Lenstra83,Kannan87,FrankTardos87}. Once the coordinates of the desired (hypersphere) center have been determined, the radius of the hypersphere can be set as the maximum Hamming distance between the center and the \blue{} vectors in $M$.  
\end{proof}
\fi

\iflong
\begin{proof}
Let $k$ be the total number of \red{} vectors plus \blue{} vectors. For convenience, we will consider the matrix representation of the input, in which the vectors are represented as the rows of a matrix $M$. 
Observe that, since there are $k$ rows of binary coordinates in $M$, the total number of different column configurations of $M$ is at most $2^k$. Let $T=\{t_1, \ldots, t_s\}$, where $s \leq 2^k$, be the set of (distinct) columns in $M$, and let $n_i$ be the number of columns in $M$ that are equal to $t_i$, for $i \in [s]$. The idea behind the \FPT{} algorithm is to encode the problem as an instance of an Integer Linear Program (ILP) with $f(k)$ variables, where $f$ is a function of $k$, and where the variables determine the coordinates of the center $\vec{c}$ of a hypersphere $B$ that contains all \blue{} vectors and excludes all the \red{} ones. It is well known that such an ILP instance can then be solved in \FPT{} time using the classical result of Lenstra~\cite{Lenstra83,Kannan87,FrankTardos87}. Once the coordinates of $\vec{c}$ have been determined, the radius of $B$ can be set as the maximum Hamming distance between $\vec{c}$ and the \blue{} vectors in $M$.  

For each column type $t_i \in T$, let $C_i$ be the set of coordinates (i.e., column indices of $M$) whose columns are equal to $t_i$. We introduce two integer-valued variables, $x_{i}^{0}, x_{i}^{1}$, where 
$x_{i}^{0}$ is the number of coordinates in $C_i$ for which $\vec{c}$ has value 0, and  $x_{i}^{1}$ is that for which $\vec{c}$ has value 1. Clearly, the total number of variables in this ILP instance is $2 \cdot s \leq 2^{k+1}$. Observe that, once $x_{i}^{0}$ and $x_{i}^{1}$ have been determined, knowing the specific $x_{i}^{0}$-many coordinates in $C_i$ for which 
$\vec{c}$ has value 0 and the specific $x_{i}^{1}$-many coordinates in $C_i$ for which $\vec{c}$ has value 1 is unnecessary. That is, any partitioning of the values of the coordinates of $\vec{c}$ that are in $C_i$ into  $x_{i}^{0}$ 0's and $x_{i}^{1}$ 1's will result in the same Hamming distance between the 
restriction of $\vec{c}$ to the coordinates of $C_i$ and the restriction of any vector in $M$ to these coordinates. This is true since all columns corresponding to the coordinates in $C_i$ are equal. More specifically, for a vector $\vec{v} \in M$, denote by $\delta_{C_i}(\vec{c}, \vec{v})$ the Hamming distance between the restriction of $\vec{c}$ and $\vec{v}$ to the coordinates in $C_i$. Then, for any partitioning of the values of the coordinates of $\vec{c}$ that are in $C_i$ into $x_{i}^{0}$ 0's and $x_{i}^{1}$ 1's, it holds that $\delta_{C_i}(\vec{c}, \vec{v})= x_{i}^{1}$ if the coordinates of $\vec{v}$ that are in $C_i$ have value 0, and $\delta_{C_i}(\vec{c}, \vec{v})= n_i-x_{i}^{1}$ if the coordinates of $\vec{v}$ that are in $C_i$ have value 1. 
 
Therefore, we can construct the desired ILP instance as follows. For each \blue{} vector $\vec{v}_{blue} \in M$ and for each \red{} vector $\vec{v}_{red} \in M$, we add the following constraint to the ILP instance, which stipulates that $\delta(\vec{c}, \vec{v}_{blue}) < \delta(\vec{c}, \vec{v}_{red})$:
\[ \sum_{i=1}^{s} \delta_{C_i}(\vec{c}, \vec{v}_{blue}) +1 \leq   \sum_{i=1}^{s} \delta_{C_i}(\vec{c}, \vec{v}_{red}).\]
Observe that the ILP instance has no objective function, and that we only need to determine the feasibility of the ILP. If no solution to the above ILP instance exists, then the algorithm reports that no solution exists for the instance of \RB. Otherwise, the solution to the above ILP instance gives the values of the variables $x_{i}^{0}, x_{i}^{1}$, for $i\in [s]$, and hence, by the above discussion, determines the coordinates of the center $\vec{c}$ of a hypersphere $B$ satisfying the constraints of the ILP instance. As such, for each \blue{} vector $\vec{v}_{blue} \in M$ and for each \red{} vector $\vec{v}_{red} \in M$, it holds that  
$\delta(\vec{c}, \vec{v}_{blue}) < \delta(\vec{c}, \vec{v}_{red})$. By choosing $r = \max \{\delta(\vec{c}, \vec{v}_{blue})  \mid \vec{v}_{blue} \in M\}$, the hypersphere $B(\vec{c}, r)$ contains all \blue{} vectors in $M$ and excludes all \red{} ones.
Since the ILP instance has at most $2^{k+1}$ variables, and hence can be solved in \FPT{} time in $k$~\cite{Lenstra83,Kannan87,FrankTardos87}, the result follows.
\end{proof}
\fi

\section{The Complexity of \RB\ with Conciseness}

In this section, we perform a detailed analysis of the complexity of \RB\ with respect to conciseness.
We will distinguish between data conciseness and
explanation conciseness. Data conciseness is the maximum number
of $1$'s appearing in any red or blue vector of the instance $I$ and is denoted $\icon(I)$; that is, $\icon(I)=\max_{\vec{v}\in
  \VR\cup \VB}\con(\vec{v})$. The explanation conciseness on the other hand
is the maximum number of $1$'s appearing in the sought-after vector $\vec{c}$.
To capture this aspect of the problem, we define a new version of \RB{} that imposes a bound $\scp$ on the explanation conciseness of the vector $\vec{c}$.
Formally, let \CRB{} be defined analogously to \RB{}, but where we are additionally given an
integer $\scp$ and the question is whether there exists a vector $\vec{c} \in D^{d}$ of conciseness at most $\scp$ and $r \in \Nat$ such that $\VB \subseteq B(\vec{c}, r)$ and $\VR \cap B(\vec{c}, r) = \emptyset$.

\subsection{Data Conciseness}

In this subsection, we analyse the parameterized complexity of \RB{}
parameterized by the conciseness of the data $\icon(I)$.
We start by showing that instances $I$ satisfying $\icon(I)\leq 3$ can be solved
in polynomial-time.
\begin{theorem}
  The restriction of \RB{} to instances $I$ satisfying
  $\icon(I)\leq 3$ can be solved in time $\bigoh(|V|d)$.
\end{theorem}

\ifshort
\begin{proof}[Proof Sketch]
{
  The main idea behind the algorithm is a case distinction
  based on the minimum distance $M_r$ of any red vector from a solution vector
  $\vec{c}$. Note first that if we fix a solution
  $\vec{c}$ for $I$, then $|\vO(\vec{v})|-2|\vO(\vec{v})\cap
  \vO(\vec{c})|$ can be used instead of the Hamming distance to compare
  the distances of two vectors from $\vec{c}$. Altogether, we obtain four cases
  for $M_r=\min_{\vec{r} \in V_r}|\vO(\vec{r})|-2|\vO(\vec{r})\cap
  \vO(\vec{c})|$
  , i.e., (1) $M_r\leq -1$, (2) $M_r=0$, (3) $M_r=1$, and (4) $M_r\geq
  2$. While (1) and (4) are trivial to solve, (2) and (3) are solved
  via a reduction to a Boolean CSP instance that can be solved in
  polynomial-time because its relational language is closed under a
  majority operation. To illustrate the ideas for (1), note that if $M_r\leq -1$,
  then it is easy to show that any solution vector must be $1$ on all
  coordinates that has a $1$ for any blue vector. But this means that
  the instance has a solution if and only if the vector $\vec{c}$ that is $1$
  at all coordinates in $\bigcup_{v \in \VB}\vO(\vec{v})$ and
  otherwise $0$ is a solution for $I$; this is because setting
  the coordinates outside of $\bigcup_{v \in \VB}\vO(\vec{v})$ to $1$
  only reduces the distance of $\vec{c}$ to vectors in $\VR$.
}
\end{proof}
\fi

\iflong
\begin{proof}
  Let $I=(\VR,\VB,d)$ be an instance of \RB{} with $\icon(I)\leq 3$ and
  let $V=\VR\cup \VB$. Note first that for one fixed solution
  $\vec{c}$ for $I$, $|\vO(\vec{v})|-2|\vO(\vec{v})\cap
  \vO(\vec{c})|$ can be used instead of the Hamming distance to compare
  the distances of two vectors from $\vec{c}$; here
  $\vec{v} \in \VR\cup \VB$.

  Let $m_r(\vec{c})=\min_{\vec{v} \in \VR}|\vO(\vec{v})|-2|\vO(\vec{v})\cap
  \vO(\vec{c})|$. Note that since $\icon(I)\leq 3$, it holds that
  $-3 \leq m_r(\vec{c})\leq 3$ for every solution $\vec{c}$ of $I$. We
  now show that for every integer $M_r$ with $-3\leq M_r\leq 3$, we
  can decide whether $I$ has a solution $\vec{c}$
  with $M_r=m_r(\vec{c})$ in time at most $\bigoh(|V|d)$.

  Note that if $\vec{c}$ is a solution for $I$ with $M_r=m_r(\vec{c})$,
  then $M_b=m_b(\vec{c})=\max_{\vec{v} \in
    \VB}|\vO(\vec{v})|-2|\vO(\vec{v})\cap \vO(\vec{c})|<M_r$, which
  will be useful in the following.
  We distinguish the following cases.
  \begin{itemize}
  \item \textbf{($\mathbf{M_r \leq -1}$ and $\mathbf{M_b \leq -2}$):}
    If $I$ contains a blue vector
    $\vec{b}$ with $|\vO(\vec{b})|=1$, we can immediately reject
    $I$; this is because for any such blue vector $\vec{b}$, we have
    $|\vO(\vec{b})|-2|\vO(\vec{b})\cap \vO(\vec{c})|\geq -1\geq M_b$.
    Similarly, it follows that $\vO(\vec{b}) \subseteq
    \vO(\vec{c})$ for every blue vector $\vec{b} \in \VB$ and any
    solution $\vec{c}$. But then, a solution $\vec{c}$ for $I$ must satisfy
    $\bigcup_{v \in \VB}\vO(\vec{v}) \subseteq \vec{c}$ and therefore
    $I$ has a solution if and only if the vector $\vec{c}$ that is $1$
    at all coordinates in $\bigcup_{v \in \VB}\vO(\vec{v})$ and
    otherwise $0$ is a solution for $I$; this is because setting
    the coordinates outside of $\bigcup_{v \in \VB}\vO(\vec{v})$ to $1$
    only reduces the distance of $\vec{c}$ to vectors in $\VR$.
    Therefore, we can solve $I$ by checking whether $\vec{c}$ is a
    solution for $I$, which can be achieved in time $\bigoh(|V|d)$.
  \item \textbf{($\mathbf{M_r = 0}$ and $\mathbf{M_b \leq -1}$):}
    Then, the following holds for every
    solution $\vec{c}$ for $I$:
    \begin{itemize}
    \item[1)] $\vO(\vec{b}) \subseteq \vO(\vec{c})$ for every blue vector
      $\vec{b} \in \VB$ with $|\vO(\vec{b})|\leq 2$; this is because
      $|\vO(\vec{b})|-2|\vO(\vec{b})\cap\vO(\vec{c})|\geq 0=M_r$, whenever
      $|\vO(\vec{b})|=1$ and $|\vO(\vec{b})\cap\vO(\vec{c})|<1$ or
      $|\vO(\vec{b})|=2$ and $|\vO(\vec{b})\cap\vO(\vec{c})|<2$.
    \item[2)] $|\vO(\vec{b})\cap\vO(\vec{c})|\geq 2$ for every blue vector $\vec{b}
      \in \VB$ with $|\vO(\vec{b})|=3$.
    \item[3)] $|\vO(\vec{r})\cap\vO(\vec{c})|\leq 0$ for every red vector $\vec{r}
      \in \VR$ with $|\vO(\vec{r})|=1$; this is because 
      $|\vO(\vec{r})|-2|\vO(\vec{r})\cap\vO(\vec{c})|<0=M_r$, if
      $|\vO(\vec{r})|=1$ and $|\vO(\vec{b})\cap\vO(\vec{c})|>0$.
    \item[4)] $|\vO(\vec{r})\cap\vO(\vec{c})|\leq 1$ for every red vector $\vec{r}
      \in \VR$ with $|\vO(\vec{r})|\geq 2$.
    \end{itemize}
    Moreover, it is easy to see that a vector $\vec{c}$ is a solution
    for $I$ if and only if it satisfies 1)--4) above. Therefore, it is
    sufficient to decide whether there is a vector that satisfies 1)--4). To achieve
    this, we will reduce the problem to a Boolean CSP instance that is
    closed under the unique Boolean majority operation
    and can therefore be solved polynomial-time.

    Let $I'=(V',C)$ be the CSP instance obtained from $I$ as follows.
    \begin{itemize}
    \item $V'$ contains one Boolean variable $v_i$ for every coordinate $i$
      with $1 \leq i \leq d$.
    \item For every blue vector $\vec{b} \in \VB$, $C$ contains the
      constraint $c_{\vec{b}}$ with scope $\{v_i \SM i \in
      \vO(\vec{b})\SE$ and relation $R_O^1$, $R_O^2$,
      and $R^3_{\geq 2}$ if $|\vO(\vec{b})|=1$, $|\vO(\vec{b})|=2$,
      and $|\vO(\vec{b})|=3$, respectively.
    \item For every red vector $\vec{r} \in \VR$, $C$ contains the
      constraint $c_{\vec{r}}$ with scope $\{v_i \SM i \in
      \vO(\vec{r})\SE$ and relation $R_Z^1$, $R^2_{\leq 1}$,
      and $R^3_{\leq 1}$ if $|\vO(\vec{r})|=1$, $|\vO(\vec{r})|=2$,
      and $|\vO(\vec{r})|=3$, respectively.
    \end{itemize}
    It is now straightforward to verify that there is a
    $d$-dimensional Boolean vector $\vec{c}$ satisfying 1)--4) above if and
    only if $I'$ has a solution. Clearly, $I'$ can be constructed from
    $I$ in time $\bigoh(|V|d)$. Moreover, because of
    Proposition~\ref{pro:solve-csp}, we can solve $I'$ in
    time $\bigoh(|V'|+|C|)=\bigoh(d+|V|)$. Therefore, we can solve $I$
    in time $\bigoh(|V|d)$.
  \item \textbf{($\mathbf{M_r = 1}$ and $\mathbf{M_b \leq 0}$):}
    The following holds for every
    solution $\vec{c}$ for $I$:
    \begin{itemize}
    \item[1)] $|\vO(\vec{b}) \cap \vO(\vec{c})|\geq 1$ for every blue vector
      $\vec{b} \in \VB$ with $|\vO(\vec{b})|\leq 2$.
    \item[2)] $|\vO(\vec{b})\cap\vO(\vec{c})|\geq 2$ for every blue vector $\vec{b}
      \in \VB$ with $|\vO(\vec{b})|=3$.
    \item[3)] $|\vO(\vec{r})\cap\vO(\vec{c})|\leq 0$ for every red vector $\vec{r}
      \in \VR$ with $|\vO(\vec{r})|\leq 2$.
    \item[4)] $|\vO(\vec{r})\cap\vO(\vec{c})|\leq 1$ for every red vector $\vec{r}
      \in \VR$ with $|\vO(\vec{b})|\geq 2$.
    \end{itemize}
    As in the previous case, it is straightforward to verify that a
    vector $\vec{c}$ is a solution
    for $I$ if and only if it satisfies 1)--4). 

    Let $I'=(V',C)$ be the CSP instance obtained from $I$ as follows.
    \begin{itemize}
    \item $V'$ contains one Boolean variable $v_i$ for every coordinate $i$
      with $1 \leq i \leq d$.
    \item For every blue vector $\vec{b} \in \VB$, $C$ contains the
      constraint $c_{\vec{b}}$ with scope $\{v_i \SM i \in
      \vO(\vec{b})\SE$ and relation $R_O^1$, $R_{\geq 1}^2$
      and $R^3_{\geq 2}$ if $|\vO(\vec{b})|=1$, $|\vO(\vec{b})|=2$,
      and $|\vO(\vec{b})|=3$, respectively.
    \item For every red vector $\vec{r} \in \VR$, $C$ contains the
      constraint $c_{\vec{r}}$ with scope $\{v_i \SM i \in
      \vO(\vec{r})\SE$ and relation $R_Z^1$, $R^2_Z$,
      and $R^3_{\leq 1}$ if $|\vO(\vec{r})|=1$, $|\vO(\vec{r})|=2$,
      and $|\vO(\vec{r})|=3$, respectively.
    \end{itemize}
    As in the previous case, it follows that $I$ is equivalent with
    $I'$ and can be solved in time $\bigoh(|V|d)$ by solving $I'$.
  \item \textbf{($\mathbf{M_r \geq 2}$ and $\mathbf{M_b \leq 1}$):}
    Then $|\vO(\vec{r})\cap \vO(\vec{c})|=0$ for every
    red vector $\vec{r} \in \VR$. Therefore, $I$ has a solution if and only if
    the vector $\vec{c}$ has $1$ at all coordinates in $[d]\setminus
    \bigcup_{\vec{r} \in V-r}\vO(\vec{r})$ and $0$ otherwise, which
    can be checked in time $\bigoh(|V|d)$.
  \end{itemize}
  Therefore, $I$ has a solution if and only if one of the above cases
  has a solution, which can be checked in time $\bigoh(|V|d)$.
\end{proof}
\fi

We now show that \RB{} is already \NP-complete for instances with $\icon(I)\geq 4$; in fact, this holds even when restricted to the class of instances where $\icon(I)$ is precisely $4$.

\begin{theorem}
\label{thm:rb-np-4}
  \RB{} is \NP-complete even when restricted to instances $I$
  satisfying $\icon(I)=4$.
\end{theorem}
We prove Theorem~\ref{thm:rb-np-4} via a reduction from the CSP problem using a constraint language
$\Gamma_4$ that is \NP-hard by Schaefer's theorem~\cite{Schaefer78,Chen09}. 
$\Gamma_4$ is the Boolean constraint language containing the
following two Boolean 4-ary relations: the \emph{red} relation $\RER$
containing all tuples having at most 2 ones and the \emph{blue}
relation $\REB$ containing all tuples having at least 3 ones.

\begin{lemma}
\label{lem:gamma}
  CSP$(\Gamma_4)$ is \NP-complete.
\end{lemma}
\iflong
\begin{proof}
  By Theorem~\ref{thm:schaefer}, it suffices to show that
  $\Gamma_4$ is not closed under any of the 4 types of operations
  stated in the theorem. The following shows that this is indeed not
  the case, i.e., $\Gamma_4$ is not closed under:
  \begin{itemize}
  \item Any constant operation since $\RER$ does not contain the
    all-one tuple and $\REB$ does not contain the all-zero tuple.
  \item Any AND operation since the AND of any two distinct tuples
    of $\REB$ each having exactly 3 ones, gives a tuple containing
    exactly 2 ones which is not in $\REB$.
  \item Any OR operation since the OR of any two distinct tuples of
    $\RER$ each having exactly 2 ones, gives a tuple containing more
    than 2 ones, which is not in $\RER$.
  \item The unique Boolean majority operation, because of
    $\RER$. Indeed, the majority of
    the following three tuples in $\RER$ is not in $\RER$: the tuples that are
    $1$ exactly at the first and second, exactly at the first and
    third, and exactly at the second and third
    entries.
  \item Any minority operation since the minority of any three distinct
    tuples in $\REB$ gives a tuple with at least $2$ ones which is not in $\REB$.\qedhere
  \end{itemize}
\end{proof}
\fi

With Lemma~\ref{lem:gamma} in hand, we establish Theorem~\ref{thm:rb-np-4} by designing a polynomial-time reduction from CSP$(\Gamma_4)$. We note that, as mentioned already in the proof of Theorem~\ref{the:2NPh}, inclusion in \NP\ is trivial.

\iflong
\begin{proof}[Proof of Theorem~\ref{thm:rb-np-4}]
\fi
\ifshort
\begin{proof}[Proof Sketch for Theorem~\ref{thm:rb-np-4}]
\fi
  Let $I=(V,C)$ be the given instance of CSP$(\Gamma_4)$. We denote by
  $C_r$/$C_b$ the set of all constraints $c$ in $C$ with
  $R(c)=\RER$/$R(c)=\REB$; note that $C=C_r\cup C_b$. We will construct
  the instance $I'=(\VR,\VB,d)$ of \RB{} as follows. First, we introduce
  one coordinate $d_v$ for every variable $v \in V$. Moreover, for
  every constraint $c \in C_r$, we introduce the red vector $\vec{r}_c$ that
  is $1$ on all coordinates that correspond to variables within the
  scope of $c$ and is $0$ otherwise, i.e., $\vec{r}_c$ is $1$ exactly on the
  coordinates in $\SB d_v \SM v \in S(c) \SE$ and $0$ at all other
  coordinates. Similarly, for
  every constraint $c \in C_b$, we introduce the blue vector $\vec{b}_c$ that
  is $1$ on all coordinates that correspond to variables within the
  scope of $c$ and $0$ otherwise.

  Finally, we will introduce two
  gadgets which will enforce that in
  every solution $\vec{c}$ of $I'$, it holds that:
  \begin{itemize}[noitemsep,topsep=0pt]
  \item[(1)] there is a red vector $\vec{r}\in \VR$ such
    that $\vec{c}$ is $1$ on at least two coordinates, where
    $\vec{r}$ is also $1$; and
  \item[(2)] there is a blue vector $\vec{b} \in \VB$ such that $\vec{c}$ is not $1$ on all
    coordinates where $\vec{b}$ is $1$.
  \end{itemize}
  Towards enforcing (1), we add two new blue vectors $\vec{v}_1^b$ and
  $\vec{v}_2^b$
  together with $8$ new coordinates $d_1^b,\dotsc,d_8^b$ such that $\vec{v}_1^b$
  is $1$ exactly at the coordinates $d_1^b,\dotsc,d_4^b$ and $\vec{v}_2^b$ is $1$
  exactly at the coordinates $d_5^b,\dotsc,d_8^b$. Moreover, for every $i$
  and $j$ with $1\leq i < j \leq 8$, we introduce a red vector
  $v^r_{ij}$ that is $1$ exactly at
  the coordinates $c_i$ and $c_j$ plus two additional fresh
  coordinates.

  Towards enforcing (2), we add one new blue vector $\vec{u}_i^b$
  together with four fresh coordinates $e_1^i,\dotsc,e_{4}^i$ for every
  $i$ with $1 \leq i \leq 4$ such
  that $\vec{u}_i^b$ is $1$ exactly on the coordinates
  $e^i_{1},\dotsc,e^i_{4}$. Finally, we add one new red vector
  $\vec{u}^r$ that is $1$ exactly at the coordinates $e_1^1$, $e^2_1$,
  $e_1^3$, and $e_1^4$.
  
  \ifshort
 We can now complete the proof by showing the equivalence between the original instance $I$ of CSP$(\Gamma_4)$ and the constructed instance $I'$ of \RB.
  \fi  
  \iflong
 To show the forward direction, let $\tau : V \rightarrow D$
  be a solution of $I$. We claim that the vector $\vec{c}$ defined by
  setting:
  \begin{itemize}
  \item $\vec{c}[d_v]=\tau(v)$ for every $v \in V$,
  \item $\vec{c}[d_1^b]=\vec{c}[d_2^b]=\vec{c}[d_3^b]=1$ and
    $\vec{c}[d_4^b]=0$,
  \item $\vec{c}[d_5^b]=\vec{c}[d_6^b]=\vec{c}[d_7^b]=1$ and
    $\vec{c}[d_8^b]=0$,
  \item $\vec{c}[e_1^i]=0$ and
    $\vec{c}[e_2^i]=\vec{c}[e_3^i]=\vec{c}[e_4^i]=1$ for every $i$
    with $1 \leq i \leq 4$, and
  \item all remaining coordinates of $\vec{c}$ are set to $0$,
  \end{itemize}
  is a solution for $I'$. We start by showing that $|\vO(\vec{c})\cap
  \vO(\vec{v})|\geq 3$ for every blue vector $\vec{v} \in \VB$. This
  clearly holds by of the definition of $\vec{c}$ and the fact
  that $\tau$ is a solution for $I$ for every blue vector $\vec{v}_c$
  with $c \in C_b$. It also clearly holds for the remaining blue
  vectors $\vec{v}^b_1$, $\vec{v}^b_2$,
  $\vec{u}^b_1,\dotsc,\vec{u}^b_4$ from the definition of $\vec{c}$.
  We now show that $|\vO(\vec{c})\cap
  \vO(\vec{v})|\leq 2$ for every red vector $\vec{v} \in \VR$, which
  shows that $\vec{c}$ is indeed a solution for $I'$. Again, this
  clearly holds for every red vector $r_c$ with $c \in C_r$ because
  $\tau$ is a solution for $I$. It also holds for the red vector
  $\vec{u}^r$ as well as all remaining red vectors $v^r_{ij}$ for
  every $i$ and $j$ with $1 \leq i < j \leq 4$ from the
  definition of $\vec{c}$.

  To show the converse, let $\vec{c}$ be a solution
  for $I'$. We claim that the assignment $\tau : V \rightarrow D$ with
  $\tau(v)=\vec{c}[d_v]$ is a solution for $I$. To show this,
  it suffices to show that (1) $|\vO(\vec{c})\cap \vO(\vec{b_c})|\geq 3$ for
  every $c \in C_b$ and (2) $|\vO(\vec{c})\cap \vO(\vec{b_c})|\leq 2$ for
  every $c \in C_r$. Towards showing (1), we will show that
  that there is a red vector $\vec{v}_{ij}^r \in \VR$  for some $i$
  and $j$ with $1 \leq i < j \leq 4$ such that
  $|\vO(\vec{c})\cap\vO(\vec{v}_{ij}^r|\geq 2$, which, since
  $\vec{c}$ is a solution for $I'$, implies that $|\vO(\vec{c})\cap
  \vO(\vec{b_c})|\geq 3$ for every $c \in C_b$.
  Suppose for a contradiction that this
  is not the case, then either $\vec{c}$ is zero at all coordinates
  $d_1^b,\dotsc,d_4^b$ or $\vec{c}$ is zero at all coordinates
  $d_5^b,\dotsc,d_8^b$. Suppose without loss of generality that the former
  holds. Then, the distance of $\vec{c}$ to $\vec{v}_1^b$ is at least as
  large as the distance of $\vec{c}$ to any red vector, contradicting
  our assumption that $\vec{c}$ is a solution for $I'$. Therefore, the
  red vector $v^r_{ij}$ and $\vec{c}$ are $1$ on at least two common
  coordinates (i.e., the coordinates $d_i^b$ and $d_j^b$).

  It remains to show (2), i.e., that $|\vO(\vec{c})\cap \vO(\vec{b_c})|\leq 2$ for
  every $c \in C_r$, for which, since $\vec{c}$ is a solution for $I$,
  it suffices to show that there is a blue vector $\vec{v}$ in $\VB$
  with $|\vec{c}\cap\vec{v}|\leq 3$. Consider the red vector
  $\vec{u}^r$. Clearly, $|\vO(\vec{c})\cap\vO(\vec{u}^r)|\leq 3$
  since otherwise $\vec{u}^r$ would be as close to $\vec{c}$ as
  every other vector, which would contradict our assumption that
  $\vec{c}$ is a solution for $I$. Therefore, there is a coordinate
  $e_1^i$ for some $1 \leq i\leq 4$ such that $\vec{c}[e_1^i]=0$. But
  then, the blue vector $\vec{u}^b_i$ satisfies that
  $|\vec{c}\cap\vec{u}^b_i|\leq 3$ and therefore $|\vO(\vec{c})\cap \vO(\vec{b_c})|\leq 2$ for
  every $c \in C_r$.
  \fi
\end{proof}

\subsection{Data Conciseness Plus $|\VR|$ or $|\VB|$}

Here we show that if in addition to the input conciseness one also parameterizes by the minimum
of the numbers of red vectors and blue vectors, then \RB{} becomes
fixed-parameter tractable.
\begin{theorem}
\label{thm:colicon}
  \RB{} is fixed-parameter tractable parameterized by
  $\icon(I)+\min\{|\VB|,|\VR|\}$.
\end{theorem}
\begin{proof}
  Let $I=(\VR,\VB,d)$ be the given instance of \RB{}. It suffices to
  show that \RB{} is fixed-parameter tractable parameterized by
  $\icon(I)+|\VB|$ and also by $\icon(I)+|\VR|$. To avoid any confusion, we remark that it is well known (and easy to see) that establishing fixed-parameter tractability w.r.t.\ the sum $\alpha+\beta$ of two numbers is equivalent to establishing fixed-parameter tractability w.r.t.\ the product $\alpha\cdot \beta$ of the same numbers.
  
  The main observation behind the algorithm (for the case
  $\icon(I)+|\VB|$) is that the total number
  of coordinates, where any blue vector can be $1$ is at most
  $|\VB|\icon(I)$; let $B=\bigcup_{\vec{b} \in \VB}\vO(\vec{b})$ be
  the set of all those coordinates. 
  Since any solution $\vec{c}$ can be assumed to be $0$ at any
  coordinate outside of $B$, we can solve $I$ by ``guessing'' (i.e., branching to find) a solution
  in time $\bigoh(2^{\icon(I)|\VB|}|V|d)$. More specifically, for
  every subset $B'$ of the at most $2^{\icon(I)|\VB|}$ subsets of $B$, we check in
  time $\bigoh(|V|d)$ whether the vector $\vec{c}$ that is $1$ exactly
  at the coordinates in $B'$ is a solution for $I$. If one of those
  vectors is a solution, then we output it, otherwise we can correctly
  return that $I$ is a \no{}-instance.
  
  The algorithm for the case where we parameterize by $\icon(I)+|\VR|$
  is almost identical with the only difference being the observation that the
  set $R=\bigcup_{\vec{r}\in \VR}\vO(\vec{r})$ has size at most
  $2^{\icon(I)|\VR|}$ and that any solution $\vec{c}$ can be assumed
  to be $1$ at every coordinate in $R$.
\end{proof}

\subsection{Explanation Concisness}

Recall that \CRB{} is defined analogously as \RB{}, but one is additionally given an
integer $\scp$ and is asked for a solution $\vec{c}$ for \RB{} with
conciseness at most $\scp$.
Note that a simple brute-force algorithm that enumerates all potential
solution vectors $\vec{c}$ with at most $\scp$ $1$'s shows that \CRB{}
is in \XP{} parameterized by $\scp$.
\begin{observation}
\label{obs:econXP}
  \CRB{} can be solved in time $\bigoh(d^{\scp}|V|d)$.
\end{observation}
Therefore, it becomes natural to ask whether this can be improved to fixed-parameter tractability. The following two theorems show that this is unlikely to
be the case, even if we additionally assume $|\VB|=1$ or $|\VR|=1$.
\begin{theorem}
\label{thm:W2_scp}
  \CRB{} is \W{2}-hard parameterized by the conciseness $\scp$ of
  the solution even if $|\VR|=1$.
\end{theorem}
\iflong
\begin{proof}
\fi
\ifshort
\begin{proof}[Proof Sketch]
\fi
  We provide a parameterized reduction from the \UHS{}
  problem, which given a set $U$ of elements, a family $\FFF \subseteq
  2^U$ of subsets of $U$ with $|F|=\ell$ for every $F \in \FFF$ and an
  integer $k$, asks whether $\FFF$ has a \emph{hitting set} $H
  \subseteq U$ of size at most $k$, i.e., $H\cap F\neq \emptyset$ for
  every $F \in \FFF$. \UHS{} is \W{2}-complete
  parameterized by $k$~\cite{DowneyF13}.

  Let $I=(U,\FFF,k)$ be an instance of \UHS{} with sets of size
  $\ell$. We construct an equivalent instance $I'=(\VR,\VB,d,\scp)$ of
  \CRB{} as follows. We set $\scp=k$. For every $u \in U$,
  we introduce the (element) coordinate $d_u$ and for every $i$ with $1 \leq i
  \leq \ell$, we introduce the (dummy) coordinate $d_i'$. Moreover, for every
  $F \in \FFF$, we add the blue vector $\vec{b}_F$ to $\VB$, which is
  $1$ on all coordinates $d_u$ with $u \in F$ and $0$ at all other
  coordinates. Finally, we introduce the red vector $\vec{r}$ that is
  $1$ at all dummy coordinates $d_i'$ and $0$ at all element
  coordinates. This completes the construction of $I'$, which can
  clearly be achieved in polynomial-time.We can now show that $I$
  is a \yes{}-instance of \UHS{} if and only if $I'$ is a 
  \yes{}-instance of \CRB{}. 
  \iflong

  Towards showing the forward direction, let $H\subseteq U$ with
  $|H|\leq k=\scp$ be a hitting set for $\FFF$. We claim that the vector
  $\vec{c}$ that is $1$ at all element coordinates $d_u$ where $u \in
  H$ and $0$ at all other coordinates, is a solution for $I'$. This
  holds since $\con(\vec{c})\leq \scp$ and
  $\delta(\vec{b}_F,\vec{c})<\ell+|H|$ for every $\vec{b}_F \in \VB{}$ but
  $\delta(\vec{r},\vec{c})=\ell+|H|$, and therefore, every blue vector in $\VB$ is
  closer to $\vec{c}$ than every red vector in $\VR$. 

  Towards showing the converse, let $\vec{c}$ with
  $\con(\vec{c})\leq \scp=k$ be a solution for $I'$. We claim that 
  the set $H$ that contains all elements $u$, where $\vec{c}[d_u]=1$ 
  is a hitting set for $\FFF$. Suppose for a contradiction that this
  is not the case and there is a set $F \in \FFF$ with $H\cap
  F=\emptyset$. Then,
  $\delta(\vec{b}_F,\vec{c})=\ell+|\vO(\vec{c})|$ 
  but also $\delta(\vec{r},\vec{c})\leq \ell+|\vO(\vec{c})|$, which contradicts our
  assumption that $\vec{c}$ is a solution for $I'$. 
    \fi
\end{proof}

\newcommand{\MIS}{\textsc{Multi-colored Independent Set}}

\begin{theorem}
\label{thm:W1_scp}
  \CRB{} is \W{1}-hard parameterized by the conciseness $\scp$ of
  the solution even if $|\VB|=1$.
\end{theorem}
\iflong
\begin{proof}
\fi
\ifshort
\begin{proof}[Proof Sketch]
\fi
  We will provide a parameterized reduction from the \MIS{}
  problem, which given an undirected graph $G=(V,E)$, where $V$ is
  partitioned into $k$ vertex sets $V_1,\dotsc,V_k$ with $|V_i|=n$ and
  $G[V_i]$ is a clique and an integer $k$, asks whether $G$ has an
  independent set of size at least $k$; note that such an independent
  set must contain exactly one vertex from each $V_i$.
  \MIS{} is well-known to be \W{1}-complete~\cite{DowneyF13}. 

  Let $I=(G,V_1,\dotsc,V_k,k)$ be an instance of \MIS{} with
  $|V_i|=n$ and $V=\bigcup_{i=1}^kV_i$. 
  We construct an equivalent instance $I'=(\VR,\VB,d,\scp)$ of
  \CRB{} as follows. We set $\scp=k$. For every $v \in V$,
  we introduce the (vertex) coordinate $d_v$ and for every $i$ with $1 \leq i
  \leq nk-2k+1$, we introduce the (dummy) coordinate $d_i'$. Moreover, for every
  $e=\{u,v\} \in E(G)$, we add the red vector $\vec{r}_e$ to $\VR$, which is
  $1$ on the coordinates $d_u$ and $d_v$ as well as the coordinate
  $d_i'$ for every $i$ with $1 \leq i \leq nk-2k+1$. We also add the red
  vector $\vec{r}$ to $\VR$ that is $1$ at the coordinates $d_i'$ with
  $1 \leq i \leq nk-2k+1$. Finally, we introduce the blue vector $\vec{b}$ that is
  $1$ at all vertex coordinates $d_v$ and $0$ at all dummy
  coordinates. This completes the construction of $I'$, which can
  clearly be achieved in polynomial-time. We can now show that $I$
  is a \yes{}-instance of \MIS{} if and only if $I'$ is a
  \yes{}-instance of \CRB{}. 
\iflong

  Towards showing the forward direction, let $C\subseteq V$ with
  $|C|=k=\scp$ be an independent set for $G$. We claim that the vector
  $\vec{c}$ that is $1$ at all vertex coordinates $d_v$ with $v \in
  C$ and $0$ at all other coordinates is a solution for $I'$. This
  holds because $\con(\vec{c})\leq \scp$ and:
  \begin{itemize}
  \item
    $\delta(\vec{r}_e,\vec{c})\geq nk-2k+1+k=nk-k+1$ for every $e\in E(G)$
    (because $G[C]$ is an independent set),
  \item $\delta(\vec{r},\vec{c})=nk-2k+1+k=nk-k+1$
  \item $\delta(\vec{b},\vec{c})=nk-k$.
  \end{itemize} 
  Therefore every red vector in $\VR$ is closer
  to $\vec{c}$ than the only blue in $\VB$.

  Towards showing the reverse direction, let $\vec{c}$ with
  $\con(\vec{c})\leq \scp=k$ be a solution for $I'$. We claim that 
  the set $C$ that contains all vertices $v$, where $\vec{c}[d_v]=1$ 
  is an independent set in $G$ of size $k$ in $G$. We first show that
  $|C|\geq k$. Suppose not, then $\delta(\vec{b},\vec{c})> nk-k$ but
  $\delta(\vec{r},\vec{c})< nk-2k+1+k=nk-k+1$, contradicting our
  assumption that $\vec{c}$ is a solution for $I$. It remains to show
  that $C$ is an independent set for $G$. Suppose for a contradiction that this
  is not the case and there is an edge $e=\{u,v\} \in E(G)$ with $u,v
  \in C$. Then,
  $\delta(\vec{r}_e,\vec{c})\leq nk-2k+1+(k-2)=nk-k-1$ 
  but also $\delta(\vec{b},\vec{c})=nk-k$, which contradicts our
  assumption that $\vec{c}$ is a solution for $I'$.
  \fi
\end{proof}

\subsection{Data and Explanation Conciseness}

As our final result in this section, we show that \CRB{} is fixed-parameter tractable when parameterized by data and explanation conciseness combined.
\begin{theorem}
\label{thm:econicon}
  \CRB{} can be solved in time $\bigoh(\icon(I)^\scp|V|d)$ and is
  therefore fixed-parameter tractable parameterized by $\scp+\icon$.
\end{theorem}
\begin{proof}
  Let $I=(\VR,\VB,d,\scp)$ with $V=\VR\cup\VB$
  be the given instance of \CRB{}. The main
  idea behind the algorithm is as follows. We start by initializing
  the solution vector $\vec{c}$ to the all-zero vector. We then check
  in time $\bigoh(|V|d)$ whether $\vec{c}$ is already a solution. If
  so, we are done. Otherwise, there must exist a red
  vector $\vec{r} \in \VR$ and a blue vector $\vec{b}\in \VB{}$ such
  that $\delta(\vec{r},\vec{c})\leq \delta(\vec{b},\vec{c})$ and
  therefore:
  $|\vO(\vec{r})|+|\vO(\vec{c})|-2|\vO(\vec{r})\cap\vO(\vec{c})|\leq
  |\vO(\vec{b})|+|\vO(\vec{c})|-2|\vO(\vec{b})\cap\vO(\vec{c})|$, or in short
  $|\vO(\vec{r})|-2|\vO(\vec{r})\cap\vO(\vec{c})|\leq
  |\vO(\vec{b})|-2|\vO(\vec{b})\cap\vO(\vec{c})|$. It follows that any
  vector $\vec{c}'$ with $\vO(\vec{c})\subseteq \vO(\vec{c}')$ and
  $\delta(\vec{r},\vec{c}')>\delta(\vec{b},\vec{c}')$ has to be
  obtained from $\vec{c}$ by flipping at
  least one coordinate in $B=\vO(\vec{b})\setminus (\vO(\vec{r})\cup
  \vO(\vec{c}))$ from
  $0$ to $1$; note that $|B|\leq \icon(I)$.
  We can therefore branch on the coordinates of $B$, and for
  every such choice $b \in B$, we continue with the vector $\vec{c}'$
  obtained from $\vec{c}$ after flipping the coordinate $b$ from $0$
  to $1$. We stop if either we have reached a solution or if the number
  of $1$'s in the current vector $\vec{c}$ exceeds the 
  conciseness upper bound $\scp$. In other words, we can solve the problem using a
  branching algorithm that has at most $|\vO(\vec{b})|\leq \icon(I)$
  many choices per branch, uses time $\bigoh(|V|d)$ per search-tree node,
  and makes at most $\scp$ branching decisions before it
  stops. Therefore, the run-time of the algorithm is $\bigoh(\icon(I)^\scp|V|d)$.
\end{proof}

\section{A Treewidth-Based Algorithm for \RB}
\label{sec:tw}

Let the \emph{incidence graph} $\incgraph$ of an instance  $I=(\VR, \VB, d)$ of \RB\ be the bipartite graph defined as follows. First of all, $V(\incgraph) = \VR\cup \VB\cup [d]$. As for the edge set, there is an edge $\vec{v}c\in E(\incgraph)$ between a vector $\vec{v}\in \VR\cup \VB$ and a coordinate $c\in [d]$ if and only if $c\in \vO(\vec{v})$. 
We identify the vertices of $\incgraph$ with the vectors in $\VR\cup\VB$ and the coordinates in $[d]$. That is, for a set of vertices $X$ in $\incgraph$, we often say ``vectors in $X$'' or ``coordinates in $X$'' to mean the vectors/the coordinates associated with the vertices in~$X$.

This section is dedicated to proving the following technical theorem, which implies all the claimed tractability results concerning the treewidth of the incidence graph:

\begin{theorem}
\label{thm:tw}
Given an instance $I=(\VR, \VB, d)$ of \RB{} and a nice tree-decomposition $\mathcal{T}=(T,\chi)$ of $\incgraph$ of width $w$, there is an algorithm solving $I$ in time $(2\min\{\scpm, \icon(I)\})^{2w+2}\cdot (|V|+d)^{\bigoh(1)}$, where $\scpm$ is the minimum conciseness of any center. Moreover, if $I$ is \yes-instance, then the algorithm outputs a center with conciseness $\scpm$ and minimum radius among all such centers. 
\end{theorem}

\ifshort
\begin{proof}[Proof Sketch]
We begin by enumerating each choice of center conciseness $\scptw = 0,1,2,\ldots, d$ and radius $r = 0,1,2,\ldots, d$, and aim to construct a solution with exactly this conciseness and radius.
\fi
\iflong
For the rest of this section, we will fix a radius $r$ and center-conciseness $\scptw$, and assume that we are looking for a center with conciseness precisely $\scptw$ such that the hypersphere around the center with radius $r$ separates the vectors in $\VB$ from the ones in $\VR$. 
To do this we will, for every possible center conciseness $\scptw = 0,1,2,\ldots, d$ in the ascending order, run the algorithm for every 
radius $r = 0,1,2,\ldots, d$. This way, we output the center with minimum conciseness and minimum radius among all centers with minimum conciseness.

\fi 
The algorithm is a bottom-up dynamic programming along the nice tree-decomposition $\mathcal{T}$. We first describe the records that we need to compute for every node $t$ of $T$. Given the description of the records, we need to show that for each of the node types  (i.e., leaf/introduce/forget/join), we can compute the records from the records of their children. Finally, we need to also show that given the records for the root node of the tree-decomposition, we can decide whether $I$ is a \yes-instance and if so output a center vector $\vec{c}$ such that $|\vO(\vec{c})| = \scptw$, $\VB\subseteq B(\vec{c}, r)$, and $\VR \cap B(\vec{c}, r) = \emptyset$. 

We begin by describing the record $\Gamma_t$ for each node $t\in T$.
We can think about $\Gamma_t$ as a map that maps a tuple $\CCC=(\pastCenter, \futureCenter, \bagCenter, \bagVectors)\in \Nat\times \Nat\times 2^{\chi(t)}\times \Nat^{|\chi(t)|}$ to either a vector $\vec{c}_{t}=\{0,1\}^{d}$ with $\con(\vec{c}_t)=\scptw$ or $\noInst$. The intuition behind the record is that the tuple $(\pastCenter, \futureCenter, \bagCenter, \bagVectors)$ is mapped to an arbitrary vector $\vec{c}_t$ such that 
\begin{enumerate}[noitemsep,nolistsep]
\item $\pastCenter$ is the number of non-zero coordinates of $\vec{c}_t$ on already ``forgotten'' coordinates, i.e., $\pastCenter = |\vO(\vec{c}_t)\cap \chi(T_t)\setminus \chi(t)|$;\label{prop:tw_1}
\item $\futureCenter$ is the number of non-zero coordinates of $\vec{c}_t$ on coordinates that are not yet introduced, i.e., $\futureCenter = |\vO(\vec{c}_t)\cap [d]\setminus \chi(T_t)|$;\label{prop:tw_2}
\item $ \bagCenter = \vO(\vec{c}_t)\cap \chi(t)$;\label{prop:tw_3}
\item $\bagVectors$ contains, for every vector $\vec{v}\in \chi(t)$, the number of ones on "forgotten" coordinates in $\vO(\vec{c}_t)$, that is $\bagVectors(\vec{v}) = |\vO(\vec{c}_t)\cap \vO(\vec{v}) \cap (\chi(T_t)\setminus \chi(t))|$;\label{prop:tw_4}
\item no forgotten \red{} vector is at distance at most $r$ from $\vec{c}_{t}$, i.e., $\VR\cap (\chi(T_t)\setminus \chi(t)) \cap B(\vec{c}_t, r)=\emptyset$; and\label{prop:tw_5}
\item all forgotten \blue{} vectors are at distance at most $r$ from $\vec{c}_{t}$, i.e., $(\VB\cap (\chi(T_t)\setminus \chi(t))) \subseteq  B(\vec{c}_t, r)$.\label{prop:tw_6}
\end{enumerate}
We say that a vector $\vec{c}_t$ that satisfies all the above properties is \emph{compatible} with $(\pastCenter, \futureCenter, \bagCenter, \bagVectors)$ for $t$. 
Moreover, $(\pastCenter, \futureCenter, \bagCenter, \bagVectors)$ is mapped to $\noInst$ if and only if no vector in $\{0,1\}^{d}$ is compatible with $(\pastCenter, \futureCenter, \bagCenter, \bagVectors)$. 

First note that if $t$ is the root node, then $\chi(t)$ is empty and $\chi(T_t)\setminus \chi(t)$ contains all vectors in the instance. Hence, if any tuple is mapped to a vector in the root, then the vector is a solution by properties~\ref{prop:tw_5}~and~\ref{prop:tw_6} above. 

We say that a tuple  $\CCC=(\pastCenter, \futureCenter, \bagCenter, \bagVectors)$ is \emph{achievable for $\Gamma_t$} if the following holds:
\begin{itemize}[noitemsep,nolistsep]
\item $\pastCenter+\futureCenter+|\bagCenter| = \scptw$; and
\item for all vectors $\vec{v}\in \chi(t)$: $\bagVectors(\vec{v})\le \min\{\scptw, |\vO(\vec{v})\cap (\chi(T_t)\setminus \chi(t))|\}$.
\end{itemize}
Note that if $\CCC$ is not achievable for $\Gamma_t$, then no vector with conciseness $\scptw$ can be compatible with $\CCC$. Hence,
the table $\Gamma_t$ will only contain the achievable tuples for $\Gamma_t$. 
\ifshort
We observe that $|\Gamma_t|\leq \scptw^2\cdot 2^{\chi(t)}\cdot (\min\{\scptw, \icon(I)\})^{|\chi(t)|}$. We can now compute the records in a leaf-to-root fashion at each of the four different types of nodes in $\mathcal{T}$.
\end{proof}
\fi
\iflong
It is straightforward to observe that: 
\begin{observation}\label{obs:records_size}
$|\Gamma_t|\le \scptw^2\cdot 2^{\chi(t)}\cdot (\min\{\scptw, \icon(I)\})^{|\chi(t)|}$, and we can enumerate all achievable tuples for $\Gamma_t$ in $\bigoh(\scptw^2\cdot 2^{\chi(t)}\cdot (\min\{\scptw, \icon(I)\})^{|\chi(t)|})$ time.
\end{observation}

\begin{lemma}[leaf node]
\label{lem:leaf_node}
Let $t\in V(T)$ be a leaf node and \(\mathcal{C} =(\pastCenter, \futureCenter, \bagCenter, \bagVectors)\) an achievable tuples for $t$. 
Then in  \(\mathcal{O}(d)\) time, we can either compute a compatible vector for \(\mathcal{C}\) or decide that no such vector exists.
\end{lemma}
\begin{proof}
Since $\chi(t)=\emptyset$, the only achievable tuple for $\Gamma_t$ is $(0, \scptw, \emptyset, \emptyset)$. One can easily verify that any vector with precisely $\scptw$ many ones is compatible with $(0, \scptw, \emptyset, \emptyset)$ for $t$.
\end{proof}

\begin{lemma}[introduce node]
\label{lem:introduce_node}
Let $t\in V(T)$ be an introduce node with child $t'$ such that $\chi(t)\setminus \chi(t')=\{v\}$ and \(\mathcal{C} =(\pastCenter, \futureCenter, \bagCenter, \bagVectors)\) an achievable tuple for $t$. 
Given $\Gamma_{t'}$, in polynomial time we can either compute a compatible vector for \(\mathcal{C}\) or decide that no such vector exists.
\end{lemma}
\begin{proof}
We distinguish between two possibilities depending on whether vertex $v$ represents a vector or a coordinate. 
First, consider that $v$ is the coordinate $c\in [d]$ and consider the tuple $\CCC' = (\pastCenter, \futureCenter + |\bagCenter\cap \{c\}|, \bagCenter\setminus\{c\}, \bagVectors)$ for the node $t'$. It is straightforward to verify that if $\vec{c}_t$ is compatible with $\CCC$ for $t$, then it is also compatible with $\CCC'$ for $t'$. 
To see this, observe that $\chi(T_t)\setminus\chi(t) = \chi(T_{t'})\setminus\chi(t')$. Hence, $\pastCenter= |\vO(\vec{c}_t)\cap \chi(T_t)\setminus \chi(t)| = |\vO(\vec{c}_t)\cap \chi(T_{t'})\setminus \chi(t')|$ and property~\ref{prop:tw_1} follows. Moreover, $\VR\cap (\chi(T_{t'})\setminus \chi({t'})) \cap B(\vec{c}_t, r) = \VR\cap (\chi(T_t)\setminus \chi(t)) \cap B(\vec{c}_t, r)=\emptyset$ and $(\VB\cap (\chi(T_{t'})\setminus \chi({t'})) = (\VB\cap (\chi(T_t)\setminus \chi(t))) \subseteq  B(\vec{c}_t, r)$ and properties~\ref{prop:tw_5}~and~\ref{prop:tw_6} follow. 
The properties \ref{prop:tw_2} is satisfied, because $\vec{c}_t$ is compatible with $\CCC$ for $t$ and so the number of non-zero coordinates of $\vec{c}_t$ in $[d]\setminus \chi(T_{t'}) = ([d]\setminus \chi(T_{t}))\cup \{c\}$ is either $\futureCenter$ if $c\notin \vO(\vec{c}_t)$ or $\futureCenter+1$. Similarly, since $\vec{c}_t$ is compatible with $\CCC$ for $t$, property \ref{prop:tw_3} is satisfied, because $\vO(\vec{c}_t)\cap \chi(t') = \vO(\vec{c}_t)\cap (\chi(t)\setminus \{c\}) = \bagCenter\setminus\{c\}$.
Finally, since $\chi(t)$ and $\chi(t')$ differ only by a single coordinate, the set of vectors is the same in both, i.e., $\chi(t)\cap (\VR\cup\VB) = \chi(t')\cap (\VR\cup\VB)$ and for every vector $\vec{v}\in \chi(t)$ we have $\bagVectors(\vec{v})=|\vO(\vec{c}_t)\cap \vO(\vec{v}) \cap (\chi(T_t)\setminus \chi(t))| = |\vO(\vec{c}_t)\cap \vO(\vec{v}) \cap (\chi(T_{t'})\setminus \chi(t'))|$ and property~\ref{prop:tw_4} follows.
Therefore, if $\Gamma_{t'}[\CCC']=\bot$, then $\Gamma_{t}[\CCC]=\bot$ and we are done.

Now let $\Gamma_{t'}[\CCC']= \vec{c}_{t'}$. If $c$ is either in both $\bagCenter$ and $\vO(\vec{c}_{t'})$ or in neither of them, then $\vec{c}_{t'}$ is compatible with $\CCC$ and we are done. Otherwise, $c\in \bagCenter$, but $c\notin \vO(\vec{c}_{t'})$. Let $c'$ be an arbitrary coordinate that has not yet been introduced such that $\vec{c}_{t'}$ is one on that coordinate; that is, $c'\in \vO(\vec{c}_{t'})\cap ([d]\setminus \chi(T_t))$. Note that such a coordinate exists since $\CCC'$ is compatible with $\vec{c}_{t'}$, and hence, by property~\ref{prop:tw_2}, $\futureCenter + |\bagCenter\cap \{c\}| = |\vO(\vec{c}_{t'})\cap [d]\setminus \chi(T_{t'})|$. Moreover, $\futureCenter + |\bagCenter\cap \{c\}|\ge |\bagCenter\cap \{c\}|=1$ and $c$ is not in $\vO(\vec{c}_{t'})$. 
We claim that the vector $\vec{c}_{t}$ such that $\vO(\vec{c}_{t}) = (\vO(\vec{c}_{t'})\setminus \{c'\})\cup \{c\}$ is compatible with $\CCC$. Properties \ref{prop:tw_1}-\ref{prop:tw_4} are straightforward to verify from the fact that $\vec{c}_{t'}$ is compatible with $\CCC'$ for $t'$. For properties~\ref{prop:tw_5} and \ref{prop:tw_6}, recall that $\chi(t')$ is a separator in $\incgraph$, and hence all forgotten vectors are zero on both $c$ and $c'$. 

Second, let the introduced vertex $v$ be a vector denoted $\vec{v}$. Note that since $\chi(t')$ is a separator in $\incgraph$, there is no edge between any forgotten coordinate $c\in [d]\cap (\chi(T_{t'}\setminus \chi(t')))$ and the vector $\vec{v}$, which means that $\vec{v}$ is zero on all the forgotten coordinates. By definition, since $\CCC$ is achievable for $t$, we have $\bagVectors(\vec{v})=0$. 
Consider a tuple \(\CCC' =(\pastCenter, \futureCenter, \bagCenter, \bagVectors')\), where $\bagVectors'(\vec{u}) = \bagVectors(\vec{u})$ for every vector $\vec{u}\in \chi(t')$. It is easy to verify that any vector $\vec{c}_t$ is compatible with $\CCC$ for $t$ if and only if it is compatible with $\CCC'$ for $t'$. Hence, we let $\Gamma_t[\CCC] = \Gamma_{t'}[\CCC']$.
\end{proof}

\begin{lemma}[forget node]
\label{lem:forget_node}
Let $t\in V(T)$ be a forget node with child $t'$ such that $\chi(t) = \chi(t')\setminus \{v\}$ and \(\mathcal{C} =(\pastCenter, \futureCenter, \bagCenter, \bagVectors)\) an achievable tuple for $t$. 
Given $\Gamma_{t'}$, in polynomial time we can either compute a compatible vector for \(\mathcal{C}\) or decide that no such vector exists.
\end{lemma}
\begin{proof}
We again distinguish between two possibilities depending on whether vertex $v$ represents a vector or a coordinate. 
First, let us assume that the forgotten vertex $v$ is a coordinate $c\in [d]$. Consider the following two tuples achievable for $\Gamma_t$, 
$\CCC'_1 = (\pastCenter, \futureCenter, \bagCenter, \bagVectors)$ and $\CCC'_2 = (\pastCenter-1, \futureCenter, \bagCenter\cup\{c\}, \bagVectors')$, where $\bagVectors'(\vec{v}) = \bagVectors(\vec{v}) - |\vO(\vec{v})\cap\{c\}|$ for all $\vec{v}\in \chi(t)$. It is straightforward to verify that if vector $\vec{c}_t$ is compatible with $\CCC$ for $t$, then either $c\notin\vO(\vec{c}_t)$ and $\vec{c}_t$ is compatible with $\CCC'_1$, or  $c\in\vO(\vec{c}_t)$ and $\vec{c}_t$ is compatible with $\CCC'_2$. Moreover, any vector compatible with $\CCC'_1$ or $\CCC'_2$ for $t'$, then it is also compatible with $\CCC$ for $t$. Therefore, we let $\Gamma_t[\CCC]= \Gamma_{t'}[\CCC_1']$ if $\Gamma_{t'}[\CCC_1']\neq \bot$, otherwise, we let $\Gamma_t[\CCC]= \Gamma_{t'}[\CCC_2']$.

Second, assume that $v$ is a vector $\vec{v}$. Consider all tuples $\CCC'_i$, for $i\in [\pastCenter]$ such that $\CCC'_i= (\pastCenter, \futureCenter, \bagCenter, \bagVectors^i)$, where $\bagVectors^i(\vec{u}) = \bagVectors^i(\vec{u})$ for all $\vec{u}\in \chi(t)$ and $\bagVectors^i(\vec{v}) = i$. It is easy to see that if $\vec{c}_t$ is compatible with $\CCC$ for $t$, then $\vec{c}_t$ is compatible with $\CCC'_i$ for $t'$ for some $i\in [\pastCenter]$. Note that since $\chi(t)$ is a separator, $\vec{v}$ is zero on all coordinates that have not been introduced yet. 
Therefore, the distance between $\vec{v}$ and any center compatible with $\CCC'_i$ is precisely
$\scptw + \con(\vec{v}) - 2\cdot (\bagVectors^i(\vec{v}) + |\vO(\vec{v})\cap \bagCenter|) = \scptw + \con(\vec{v}) - 2\cdot (i + |\vO(\vec{v})\cap \bagCenter|)$, and in particular, this distance does not depend on the center as long as the center is compatible with $\CCC'_i$ for $t'$. For that reason, we just need to go over all $i\in [\pastCenter]$, and check if $\Gamma_{t'}[\CCC'_i]\neq \bot$ and if so we check if $\scptw + \con(\vec{v}) - 2\cdot (i + |\vO(\vec{v})\cap \bagCenter|)$ is at most $r$ in case $\vec{v}$ is \blue, or at least $r+1$ if $\vec{v}$ is \red. If both conditions are satisfied for some $i\in [\scptw]$, then we let $\Gamma_t[\CCC] = \Gamma_{t'}[\CCC'_i]$; otherwise, we let $\Gamma_t[\CCC] = \bot$.
\end{proof}

\begin{lemma}[join node]
\label{lem:join_node}
Let $t\in V(T)$ be a join node with children $t_1, t_2$ such that $\chi(t) = \chi(t_1)=\chi(t_2)$ and \(\mathcal{C} =(\pastCenter, \futureCenter, \bagCenter, \bagVectors)\) an achievable tuple for $t$. 
Given $\Gamma_{t_1}$ and $\Gamma_{t_2}$, in \(\mathcal{O}(|\Gamma_{t_1}|\cdot(|V|+d)^{\bigoh(1)})\) time we can either compute a compatible vector for \(\mathcal{C}\) or decide that no such vector exists.
\end{lemma}
\begin{proof}
Let $\vec{c}_t$ be a vector compatible with $\CCC$ for $t$. Observe that $\vec{c}_t$ is also compatible with some tuple $\CCC_1 = (\pastCenter^1, \futureCenter^1, \bagCenter^1, \bagVectors^1)$ for $t_1$ and some tuple $\CCC_2 = (\pastCenter^2, \futureCenter^2, \bagCenter^2, \bagVectors^2)$ for $t_2$ such that 
\begin{itemize}[noitemsep,nolistsep]
\item $\pastCenter = \pastCenter^1+\pastCenter^2$;
\item $\futureCenter^i = \scptw - \pastCenter^i - |\bagCenter| = \futureCenter+\pastCenter^{3-i}$ for $i\in \{1,2\}$; 
\item $\bagCenter = \bagCenter^1 = \bagCenter^2$; and
\item $\bagVectors(\vec{v}) =\bagVectors^1(\vec{v}) + \bagVectors^2(\vec{v})$ for all vectors $\vec{v}\in \chi(t)$.
\end{itemize}
We call such pair of tuples $\CCC_1$ and $\CCC_2$ satisfying the above four conditions \emph{joinable} pair.

On the other hand, consider a joinable pair of tuples $\CCC_1 = (\pastCenter^1, \futureCenter^1, \bagCenter^1, \bagVectors^1)$ for $t_1$ and $\CCC_2 = (\pastCenter^2, \futureCenter^2, \bagCenter^2, \bagVectors^2)$. Moreover, let $\vec{c}^1_t$ be a vector compatible with $\CCC_1$ for $t_1$ and  $\vec{c}^2_t$ be a vector compatible with $\CCC_2$ for $t_2$. 
Now let $\vec{c}_t$ be an arbitrary vector such that $\vO(\vec{c}_t)\cap \chi(T_{t_1})\setminus \chi(t) = \vO(\vec{c}^1_t)\cap \chi(T_{t_1})\setminus \chi(t)$, $\vO(\vec{c}_t)\cap \chi(T_{t_2})\setminus \chi(t) = \vO(\vec{c}^2_t)\cap \chi(T_{t_2})\setminus \chi(t)$, $\vO(\vec{c}_t)\cap \chi(t) = \bagCenter$, and $|\vO(\vec{c}_t)\setminus \chi(T_t)| = \futureCenter$. That is, $\vec{c}_t$ agrees with $\vec{c}_t^1$ on the coordinates that are forgotten in a node below $t_1$, and with $\vec{c}_t^2$ on the coordinates that are forgotten in a node below $t_2$. We claim that $\vec{c}_t$ is compatible with $\CCC$ for $t$. Properties \ref{prop:tw_1}-\ref{prop:tw_4} follow directly from the construction of $\vec{c}_t$. Now let $\vec{v}$ be a vector in $\chi(T_{t_i})\setminus \chi(t_i)$ for $i\in \{1,2\}$. We claim that the distance between $\vec{v}$ and $\vec{c}_t$ is precisely the same as the distance between $\vec{v}$ and $\vec{c}_t^i$. This is because $\vec{c}_t^i$ and $\vec{c}_t$ coincide on the coordinates in $\chi(T_{t_i})$, $\vec{v}$ is zero on all the other coordinates, and the number of non-zero coordinates of $\vec{c}_t^i$ and $\vec{c}_t$ in $[d]\setminus \chi(T_{t_i})$ is the same. It follows that properties~\ref{prop:tw_5}~and~\ref{prop:tw_6} are satisfied as well. 

Therefore, to compute $\Gamma_t[\CCC]$ we can go over all tuples $\CCC_1 = (\pastCenter^1, \futureCenter^1, \bagCenter^1, \bagVectors^1)$ in $\Gamma_{t_1}$ such that $\bagCenter^1 = \bagCenter$.
By definition of ``joinable pair of tuples'', there is a unique tuple $\CCC^2$ that forms a joinable pair with $\CCC^1$ given by $\CCC_2 = (\pastCenter-\pastCenter^1, \futureCenter+\pastCenter^1, \bagCenter, \bagVectors^2)$, where $\bagVectors^2(\vec{v}) = \bagVectors(\vec{v})-\bagVectors^1(\vec{v})$ for all vectors $\vec{v}\in \chi(t)$. 
 If $\Gamma_{t_1}[\CCC_1]\neq \bot$ and $\Gamma_{t_2}[\CCC_2]\neq \bot$, then we compute $\Gamma_t[\CCC]$ as above and stop. On the other hand, if for all pairs $(\CCC_1, \CCC_2)$ we have either $\Gamma_{t_1}[\CCC_1] =  \bot$ or $\Gamma_{t_2}[\CCC_2] = \bot$, then we correctly return that  $\Gamma_t[\CCC] = \bot$. 
\end{proof}

Now we are ready to put the whole algorithm together and prove the main theorem of this section. 

\begin{proof}[Proof of Theorem~\ref{thm:tw}]
The algorithm works by going over all $\scptw\in \{0,1, \ldots, d\}$, and for each $\scptw$, it goes over all $r\in \{0,1, \ldots, d\}$. For each pair $(\scptw, r)$, we use the algorithms of Observation~\ref{obs:records_size} and Lemmas~\ref{lem:leaf_node}-\ref{lem:join_node} to compute the table $\Gamma_t$ for each node $t\in T$. Finally, we go over all tuples for the root node $t_r$. If for some tuple $\CCC$ we have $\Gamma_{t_r}[\CCC]\neq \bot$, we return $\Gamma_{t_r}[\CCC]$, note that in this case $\scpm = \scptw$ is indeed a minimum possible conciseness, otherwise, we continue to next pair $(\scptw, r)$. The running time follows from Observation~\ref{obs:records_size} and Lemmas~\ref{lem:leaf_node}-\ref{lem:join_node}.
\end{proof}
\fi

Combining Theorem~\ref{thm:tw} and Proposition~\ref{fact:findtw}, we get the following three corollaries: 

\begin{corollary}\label{cor:tw}
\RB{} and \CRB{} are in \XP{} parameterized by $\tw(\incgraph)$.
\end{corollary}

\begin{corollary}\label{cor:tw_econ}
\CRB{} is fixed-parameter tractable parameterized by $\tw(\incgraph)+\scp$.
\end{corollary}

\begin{corollary}\label{cor:tw_dcon}
\RB{} and \CRB{} are fixed-parameter tractable parameterized by $\tw(\incgraph)+\icon(I)$.
\end{corollary}
\section{Concluding Remarks}
\label{sec:conclusion}

In this  paper, we studied hypersphere classification problems from a
parameterized complexity perspective, focusing strongly on
conciseness. We considered conciseness in terms of the sought-after
explanation and in terms of the feature vectors in the training data.
Our algorithmic and lower-bound results draw a comprehensive
complexity map of hypersphere classification. This map pinpoints the
exact complexity of the various combinations of parameters which can
either measure the structural properties of the input data or the
conciseness of data or explanations.

All our lower and upper complexity bounds are essentially tight, with a single exception:
While we show that hypersphere classification without conciseness
restrictions is XP-tractable when parameterized by treewidth alone, whether the problem is fixed-parameter tractable or \W{1}-hard under this parameterization is left open.

Finally, we remark that all our results carry over to the case where
one aims to find a minimum-radius separating hypersphere (instead of
merely deciding whether one exists) that classifies the training data.
This problem has also been extensively
studied~\cite{pcooper,neskovic,neskovic1,astorino,astorino1}.

\section*{Acknowledgments}

Robert Ganian acknowledges support from the Austrian Science Fund
(FWF, project Y1329). Iyad Kanj acknowledges support from DePaul
University through a URC grant 602061.  Stefan Szeider acknowledges
support from the Austrian Science Fund (FWF, projects P32441), and
from the Vienna Science and Technology Fund (WWTF, project ICT19-065).
Sebastian Ordyniak acknowledges support from the
Engineering and Physical Sciences Research Council (EPSRC, project EP/V00252X/1).

\bibliographystyle{icml2023}
\bibliography{literature}
 
\end{document}